\documentclass{article}
\usepackage{amsthm, amssymb, amsmath,verbatim}
\usepackage[margin=1in]{geometry}
\usepackage{enumerate}
\usepackage{graphicx}
\graphicspath{ {images/} }
\usepackage{amsmath}

\newcommand{\E}{\mathbb{E}}
\newcommand{\genstirlingI}[3]{%
  \genfrac{[}{]}{0pt}{#1}{#2}{#3}%
}
\newcommand{\stirling}[2]{\genstirlingI{}{#1}{#2}}
\newcommand{\dtv}{d_{TV}}
\newcommand{\eps}{\epsilon}
\newcommand{\pr}{\mathrm{Pr}}

\newcommand\blfootnotea[1]{%
  \begingroup
  \renewcommand\thefootnote{}\footnote{#1}%
  \endgroup
}

\newtheorem*{claim}{Claim}
\newtheorem{thm}{Theorem}
\newtheorem{lem}[thm]{Lemma}

\newtheorem{defn}[thm]{Definition}
\newtheorem{cor}[thm]{Corollary}
\newtheorem{fact}[thm]{Fact}
\newtheorem{prop}[thm]{Proposition}

\title{New Lower Bounds for Testing Monotonicity and Log Concavity of Distributions \blfootnotea{Authors are listed in randomized order.}}

\author{
Yuqian Cheng \\  {\tt chengyuqian6@gmail.com} \and
	Daniel M. Kane\thanks{Supported by NSF Medium Award CCF-2107547,
NSF Award CCF-1553288 (CAREER), a Sloan Research Fellowship, and a grant from CasperLabs.}\\
	UC San-Diego \\
	{\tt dakane@ucsd.edu}\\
	\and
Zhicheng Zheng \\ New York University \\ {\tt zz4230@nyu.edu}
}

\begin{document}
\maketitle
\begin{flushleft}
\section{Introduction}
\subsection{Background}

Given data from an unknown distribution, can one determine whether the underlying distribution has certain properties or not? In order to make this determination, how much data would be needed? These are the critical questions in the field of statistical hypothesis testing (see \cite{hypothesistesting}). Some of these questions, particularly when dealing with discrete distributions, have over the past few decades drawn the interest of computer scientists under the heading of distribution testing (see \cite{distributiontesting}).

In order to make such a determination possible, distribution testing algorithms are usually asked to distinguish between the cases where the unknown distribution $p$ either has the property $\mathcal{P}$ or is far from (usually in $L^1$ distance) any distribution with this property. The goal is usually to do this with as few samples as possible (ideally to within a constant factor of the information-theoretic limits) and in a computationally efficient manner. While many of the most basic questions such as testing for uniformity, identity, closeness and independence have all been resolved, many more complicated properties still have wide gaps between the best known algorithms and the best lower bounds. For a survey of progress in this field see \cite{survey} for a recent survey of the area.

In this paper, we develop a new lower bound technique for testing properties that are defined by inequalities between the probability masses of individual bins. In particular, this technique is used to produce new lower bounds for testing monotonicity and log-concavity, the latter of which matches known upper bounds to within polylog factors.

\subsection{Notation}

We use $[n]$ to denote the set $\{1,2,\ldots,n\}$. As we will usually be dealing in this paper with discrete distributions, for a distribution $p$ on a discrete set $S$ and an element $i\in S$, we let $p_i$ denote the probability that $p$ assigns to the element $i$. The distance considered in this paper will usually be the total variational distance denoted by $\dtv$, defined as $d_{TV}(p,q) := \frac{1}{2}||p-q||_1$. In particular, if $p$ and $q$ are distributions defined over the same discrete set $S$, $\dtv(p,q) = \frac{1}{2}\sum_{i\in S} |p_i - q_i|.$

By an \emph{ensemble} we will typically mean a probability distribution over probability distributions on some fixed set $S$.

\subsection{Our Results}

Our main applications have to do with the monotonicity test and log-concavity testing problems, so we begin by defining our terms.
\begin{defn}
We say that a distribution $p$ on $[n]$ is \emph{monotone} (decreasing) if $p_i \geq p_j$ for all $i < j$.
\end{defn}
In fact, we will also deal with monotone distributions defined over higher dimensional cubes. To make sense of this we first need to define a partial order on $[n]^d$:
\begin{defn}
For $\textbf{i},\textbf{j} \in [n]^d$, we say that $\textbf{i}<\textbf{j}$ if $\textbf{i}_a<\textbf{j}_a$ for all $1 \leq a \leq d$.

We say that a distribution $p$ on $[n]^d$ is \emph{monotone} if $p_{\textbf i} \geq p_{\textbf j}$ whenever $\textbf{i} < \textbf{j}$.
\end{defn}

We also define log-concavity as follows:
\begin{defn}
A distribution $p$ on $[n]$ is \emph{log-concave} if its support is a contiguous interval and its density function, $p_i$, satisfies that $p_i^2 \geq p_{i+1}p_{i-1}$ for all $2 \leq i \leq n-1$.
\end{defn}

Our main results are to prove new lower bounds for multidimensional monotonicity testing and log concavity testing. In particular, we show that algorithms that can reliably distinguish between a distribution having the desired property (monotonicity or log-concavity) and being $\eps$-far in total variation distance from any such distribution, must make use of a relatively large number of samples.

The result for multidimensional monotonicity testing is stated as below.
\begin{thm}\label{monotonicity thm}
For $\epsilon > 0$ sufficiently small, any algorithm that can distinguish whether the distribution over $[n]^d$ is monotone (for some $n$ at least a sufficiently large multiple of $d\log(1/\eps)$) or $\epsilon$ far from monotone with probability over $\frac{2}{3}$ requires at least $N = 2^{-O(d)}d^{-d} \eps^{-2} \log^{-7}(1/\eps) \min(n,d \eps^{-1} \log^{-3}(1/\eps))^d$ samples.
\end{thm}
We get the following lower bound of log concavity testing of distribution over $[n]$.
\begin{thm}\label{log concavity thm}
For $\epsilon > 0$ sufficiently small, let $p$ be distribution over $[n]$ where $n$ is at least a sufficiently large multiple of $\log(1/\eps)$. Any algorithm that can distinguish whether the distribution is log concave or $\epsilon$ far from log concave with probability over $\frac{2}{3}$ requires at least $N = \Omega(\log^{-7}(1/\eps) \eps^{-2} \min(n,\eps^{-1/2}\log^{-3/2}(1/\eps)))$ samples.
\end{thm}

\subsection{Prior and Related Works}

The problem of monotonicity distribution testing was first considered by \cite{bkr} who developed a tester with sample complexity $O(\frac{\sqrt{n}\log n}{\epsilon^4})$ for distributions over $[n]$. Then the result was generalized to give a tester with sample complexity $\widetilde{O}(n^{d-\frac{1}{2}}\textrm{poly}(\frac{1}{\epsilon})) $ for distributions over $[n]^d$ by \cite{BFRV}. The best currently known result is by \cite{adk}, $O(\frac{n^{\frac{d}{2}}}{\epsilon^2}+(\frac{d\log n}{\epsilon^2})^d\frac{1}{\epsilon^2})$ for testing monotonicity of a distribution over $[n]^d$. On the other hand, the best lower bounds to date for this problem come from the lower bounds for uniformity testing giving a sample complexity of $\Omega(\sqrt{n^d}/\eps^2)$. While this shows that the algorithm of \cite{adk} is asymptotically optimal for $n$ much larger than $1/\eps$, it leaves a pretty substantial gap when $\eps$ is small.

For testing log concavity of a distribution over $[n]$, \cite{adk} provides the first known tester for the low sample regime of testing log concavity, which requires $O(\frac{\sqrt{n}}{\epsilon^2}+\frac{1}{\epsilon^5})$ samples. Then \cite{cdgr} gives an improved algorithm with sample complexity $O(\frac{\sqrt{n}}{\epsilon^{\frac{7}{2}}})$. The latest result for sample complexity of log concavity testing lies in \cite{cds}, $O(\frac{\sqrt{n}}{\epsilon^2})+\widetilde{O}(\frac{1}{\epsilon^{\frac{5}{2}}})$ for testing a distribution over $[n]$. Once again, previously known lower bounds for this problem were somewhat lacking, consisting only of the $\Omega(\sqrt{n}/\eps^2)$ lower bound for uniformity testing. However, importantly, if one combines this with our lower bound, it nearly matches the best of the upper bound of \cite{cds} and the trivial $O(n/\eps^2)$ bound from learning $p$ to error $\eps$. Together these show that the optimal sample complexity for testing log-concavity is $\left( \frac{\sqrt n}{\eps^2} + \frac{\min(n,\eps^{-1/2})}{\eps^2}\right)$ up to polylogarithmic factors.

\subsection{Our Techniques}

Our techniques come from a general framework for obtaining lower bounds for testing properties defined by imposing inequalities on the individual bin probabilities. Our starting point is fairly standard for such lower bounds: we want to construct two ensembles of distributions over some support set $S$, $D_{yes}$ and $D_{no}$ where the distributions in $D_{yes}$ have the property with high probability and distributions in $D_{no}$ far from having the property with high probability. Therefore,  if an algorithm can distinguish a distribution having the property or $\epsilon$ far from having the property through $N$ samples with probability at least $\frac{2}{3}$, then it should be able to distinguish $N$ samples from a distribution in $D_{yes}$ and $N$ samples from a distribution in $D_{no}$ with probability at least $\frac{3}{5}$. However, we show that if a distribution $p$ is randomly taken from either $D_{yes}$ or $D_{no}$, then a small number $N$ of samples from $p$ will be insufficient to reliably determine which of the two ensembles it was sampled from. In particular, our goal will be to show that the two resulting distributions on $S^N$ will be close in total variational distance, making such a determination information-theoretically impossible.

In order to construct these ensembles, a key insight is that we will need them to match moments. In particular, if $n_i$ is the number of samples drawn from the $i^{th}$ bin then the expectation of $\binom{n_i}{k}$ will be $\binom{N}{k}\E[p_i^k]$. If we want our ensembles to be nearly indistinguishable, it is thus a good idea to ensure that these moments- $\E_{p\sim D}[p_i^k]$- are the same for $D_{yes}$ and $D_{no}$ for all small (in our case at most logarithmic) values of $k$. This will typically ensure that $D_{yes}$ and $D_{no}$ are hard to distinguish.

On the other hand, if the property in question is defined by inequalities among the $p_i$ (like $p_i \geq p_{i+1}$ for monotonicity or $p_i^2 \geq p_{i-1}p_{i+1}$ for log-concavity), then these kinds of inequalities are not defined by moments. In particular, our goal will be to find a moment matching pair of ensembles $D_{yes}$ and $D_{no}$ so that distributions from $D_{yes}$ satisfy these inequalities, but distributions from $D_{no}$ do not.

For the mechanics of this construction, we begin by constructing a pair of moment distributions over real numbers $F_{yes}$ and $F_{no}$ with $\E[F_{yes}^k] = \E[F_{no}^k]$ for all small natural numbers $k$. These will be used to modify the bin probabilities of a base distribution $Q$. In particular, in a sample from $D_{yes/no}$ the bin probability of a bin $p_i$ will by $Q_i + A_iF_{yes/no}$ for some carefully chosen constants $A_i$ and with the samples from $F_{yes/no}$ coupled in such a way to ensure that $p$ remains properly normalized. On the other hand, we can construct our distributions $F_{yes/no}$ so that while $F_{yes}$ is positive almost surely, $F_{no}$ has a reasonable probability of being reasonably negative, which (assuming that $Q$ was constructed carefully) will break the inequalities defining the property in question.

In Section \ref{generic construction section}, we will explain the generic version of the construction of these ensembles $D_{yes}$ and $D_{no}$ and in particular prove Proposition \ref{main bound prop} to show that they are indistinguishable with a small number of samples. The next three sections will be applications. In particular, in Section \ref{one dim monotonicity sec}, as a warmup we will prove a lower bound for one-dimensional monotonicity testing. In Section, \ref{multidim monotonicity sec} we will generalize this to a lower bound for multidimensional monotonicity testing. Finally, in Section \ref{log concavity sec} we will prove our lower bound for testing log-concavity.

\section{Generic Lower Bound Construction}\label{generic construction section}

This section contains the key construction for our lower bound technique. In particular, we provide a general framework for producing two ensembles of distributions $D_{yes}$ and $D_{no}$ over some finite set $S$ that are hard to distinguish using few samples. In particular, for a positive integer $N$, we consider the two distributions over $S^N$ which we call $D_{yes}^N$ and $D_{no}^N$ given by taking a random distribution $p$ from the ensemble $D_{yes/no}$ and then returning $N$ i.i.d. samples from $p$. The key result here is Proposition \ref{main bound prop}, where we show that as long as $N$ is not too large relative to the other parameters of the construction that $\dtv(D_{yes}^N,D_{no}^N)$ is small, thus implying that one cannot reliably determine whether $p$ was taken from $D_{yes}$ or $D_{no}$ with only $N$ samples.

\subsection{Construction of $D_{yes}$ and $D_{no}$}

In this section, we describe the general procedure for construction ensembles $D_{yes}$ and $D_{no}$. The basic idea is to start with a fixed distribution $Q$ over $S$ and to tweak it slightly. In order to ensure that these tweaks match moments and preserve indistinguishability, we will first need to find a pair of real-valued distributions $F_{yes}$ and $F_{no}$ that match their low order moments, and we will use the outputs of these distributions to tweak the bin probabilities of $Q$. In order to ensure that the resulting distribution remains properly normalized, we will pair up a number of the bins in $S$ getting pairs $(j_1,k_1),(j_2,k_2),\ldots,(j_s,k_s)$ and ensure that any probability mass taken from $j_i$ is added to $k_i$ and visa versa. Finally, to decide the amount of mass to move we will sample $\delta_i$ proportional to $F_{yes/no}$ and our final distribution will have $p_{j_i} = Q_{j_i}+\delta_i$ and $p_{k_i} = Q_{k_i}-\delta_i$.

We begin by defining our distributions $F_{yes}$ and $F_{no}$, for which we will need to have a few free parameters:
\begin{defn}
Let $m$ be an integer, $A$ and $g$ be real numbers and $a$ be a uniformly random integer$\mod{m}$, we define the probability distributions $F_{yes}^{A,g,m}$ and $F_{no}^{A,g,m}$ to be the distribution of $A(\cos(\frac{2\pi a}{m}) + g)$ and $A(\cos(\frac{2\pi (a+\frac{1}{2})}{m}) + g)$.
\end{defn}
The first critical property of these distributions is that they match their first $m-1$ moments, which we can prove by making use of the Chebyshev polynomials $T_m(\cos\theta) := \cos(m\theta)$ for $\theta \in\left[0, \pi \right]$.
\begin{lem}\label{F moment matching lem}
For any positive integer $k$ less than $m$,
$$E_{\delta \in F_{yes}^{A,g,m}}[\delta^k]=E_{\delta \in F_{no}^{A,g,m}}[\delta^k].$$
\end{lem}
\begin{proof}
Note that the roots of $T_m(x) + 1$ and $T_m(x) - 1$ are ${\cos(\frac{2\pi(a+\frac{1}{2})}{m})}_{0 \leq a < m}$ and ${\cos(\frac{2\pi a}{m})}_{0 \leq a < m}$ respectively. Since $T_m(x) + 1$ and $T_m(x) - 1$ only differ by a constant, all elementary symmetric polynomials of degree less than $m$ of roots of one agree with the corresponding polynomials of roots of the other. By the fundamental theorem on symmetric polynomials, for roots of a polynomial $r$, $\sum_{r} r^k$ can be written in terms of elementary symmetric polynomials, where $\sum_{r} r^k$ is proportional to the $k$th moment of roots. Since the roots have $m-1$ identical elementary symmetric polynomials, we can conclude that they have $m-1$ matching moments. In particular, this means that
$$
\E\left[ \cos(\frac{2\pi(a+\frac{1}{2})}{m}) \right] = \E\left[ \cos(\frac{2\pi a}{m}) \right].
$$

Applying the linear transformation $x \rightarrow A(x+g)$, we note that the distributions $A(\cos(\frac{2\pi a}{m}) + g)$ and $A(\cos(\frac{2\pi (a+\frac{1}{2})}{m}) + g)$ must also have $m-1$ matching moments as $$\E[(A(x+g))^k] = \sum_{k'=0}^k A^k\binom{k}{k'}g^{k-k'} \E[x^{k'}].$$
\end{proof}
We are now ready to define $D_{yes}$ and $D_{no}$ below:
\begin{defn}\label{D yes no def}
Suppose that we have:
\begin{itemize}
\item A distribution $Q$ over a finite set $S$.
\item A set of disjoint pairs of elements of $S$: $(j_1,k_1),(j_2,k_2),\ldots,(j_s,k_s)$.
\item A positive integer $m$.
\item Two sequences of real numbers $(A_i)_{1 \leq i \leq s}$ and $(g_i)_{1 \leq i \leq s}$ so that $|A_i| \leq \min (\frac{Q_{j_i}}{1+|g_i|}, \frac{Q_{k_i}}{1+|g_i|})$ for all $i$.
\end{itemize}
Given this, we define a pair of ensembles of distributions over $S$, $D_{yes}$ and $D_{no}$ as follows:

To select a distribution $p$ from $D_{yes/no}$, we first select $\delta_i$ independently from $F^{A_i,g_i,m}_{yes/no}$ for each $1\leq i \leq s$. For $a\in S$ with $a$ not equal to any $j_i$ or $k_i$, we let $p_a = Q_a$. Otherwise, we let $p_{j_i}=Q_{j_i}+\delta_i$ and $p_{k_i}=Q_{k_i}-\delta_i$.

\end{defn}
Note that as $|A_i| \leq \min (\frac{Q_{j_i}}{1+|g_i|}, \frac{Q_{k_i}}{1+|g_i|})$, $p_a$ in non-negative for all $a\in S$. In addition, for each $i$, $p_{j_i}+p_{k_i}=Q_{j_i}+Q_{k_i}$, from which it is not hard to see that $\sum_{a\in S}p_a$ is always $1$. These observations confirm that $p$ is in fact defines a probability distribution over $S$.

Another important remark is that conditioned on a sample from $p$ landing in the $i$th pair of bins, $(j_i,k_i)$, the probability of it landing in the $j_i$th bin depends only on the value of $\delta_i$, and $\delta_i$s are independently sampled from $F_{yes/no}^{A_i,g_i,m}$. This is a crucial condition our key proposition needs.

For this construction, we are hoping to prove the following proposition,
\begin{prop}\label{main bound prop}
Given $Q,(A_i)_{1 \leq i \leq s},(g_i)_{1 \leq i \leq s},m,(j_i,k_i)_{1 \leq i \leq s}$ be as above and let $D_{yes}$ and $D_{no}$ be as in Definition \ref{D yes no def}. Assume furthermore, that $m$ is at least a $C \log(s)$ for some sufficiently large constant $C$.

For integers $N>\frac{6\log s}{\min_{i} (Q_{j_i}+Q_{k_i})}$, we define $D_{yes}^N$ to be the distribution on $S^N$ obtained by first taking a random distribution $p$ from $D_{yes}$ and then taking $N$ independent samples from $p$, and define $D_{no}^N$ similarly. Then letting $x_{max} = \max_{1\leq i\leq s} \left(\frac{|A_i|(1+|g_i|)}{\min(Q_{k_i},Q_{j_i})} \right)$ and $B = 2\max_{1\leq i\leq s} (Q_{j_i}+Q_{k_i})N$, then if $x_{max} < 1/10$, we have that $d_{TV}(D_{yes}^N , D_{no}^N )$ is at most
$$
O(1/s) + m^4 s O(\sqrt{B \log(s)} + x_{max}B)(1+x_{max})^{O(\sqrt{B \log(s)} + x_{max}B)}O(\sqrt{x_{max}^2B \log(s)} + x_{max}^2 B)^m.
$$
\end{prop}
In our applications, we will take $s$ on the order of $|S|$ and will use $Q$'s which are not too far from uniform (and thus $\min_{a \in S} Q_a$ will be on the order of $1/s$). In order to ensure that $D_{yes}$ and $D_{no}$ perturb $Q$ by at least $\eps$ in total variational distance, we will want $x_{max}$ (which is essentially the largest relative perturbation of any bin probability) to be on the order of $\eps$. Taking $N$ on the order of $s/\eps^2$ up to some polylog factors gives us $B$ on the order of $1/\eps^2$.

From here we note that the
$$
(1+x_{max})^{O(x_{max} B + \sqrt{B \log(s)})}
$$
term is
$$
\exp(O(B x_{max}^2 + \sqrt{B x_{max}^2 \log(s)})),
$$
which is not too large. On the other hand, so long as we ensure that
$$
(x_{max}\sqrt{B \log(s)}+B x_{max}^2)
$$
is less than a sufficiently small constant and keep $m$ to be a large enough multiple of $\log(s)$, this term will dominate the things it is multiplied by, thus leaving us with a final bound that is quite small.

In particular, we have
\begin{cor}\label{main bound cor}
In the notation of Proposition \ref{main bound prop}, if we have additionally that $B x_{max}^2$ is at most a sufficiently small multiple of $1/\log(s)$, $m$ is at least a sufficiently large multiple of $\log(s/(x_{max}\eps))$ and $s$ is at least a sufficiently large constant, then
$$
d_{TV}(D_{yes}^N , D_{no}^N ) < \frac{1}{100}.
$$
\end{cor}
\begin{proof}
Assume that for some $A$ sufficiently large that $B x_{max}^2 \log(s) < 1/A$ and that $m \geq A\log(s/(x_{max}\eps))$. Then the bound in Proposition \ref{main bound prop} reduces to
$$
O(1/s) + O(m^4 s / x_{max})(\sqrt{B x_{max}^2 \log(s)} + x_{max}^2 B) \exp(O(\sqrt{B x_{max}^2 \log(s)} + x_{max}^2 B)) O(\sqrt{B x_{max}^2 \log(s)} + x_{max}^2 B)^m.
$$
Noting that $(\sqrt{B x_{max}^2 \log(s)} + x_{max}^2 B) = O(1/A)$, this reduces to
$$
O(1/s) + O(m^4 s/ (A x_{max})) \exp(O(1/A)) O(1/A)^m.
$$
In particular, if $A$ is large enough the $O(1/A)^m$ term is at most $(1/2)^m$, which if $m$ is a sufficiently large multiple of $\log(s/(x_{max}\eps))$ is at most $m^{-4} x_{max}/s^2$, which would make our final bound $O(1/s)$. If $s$ is sufficiently large, this is less than $1/100.$
\end{proof}


\subsection{Comparison of Distributions of Number of Samples in Bins $j_i$ and $k_i$}

In our construction of $D_{yes}$ and $D_{no}$ we refer to the $i$th pair of bins as the pair $\{j_i,k_i\}$. As $D_{yes}$ and $D_{no}$ are essentially identical except in how they distribute probability mass between the $i$th pair of bins for various values of $i$, in order to show that they are hard to distinguish, it will be important for us to show that the distribution on the number of samples in these bins is close for $D_{yes}$ and $D_{no}$. In particular, if we condition on the number of samples $B_i$ that land in the $i$th pair of bins, and consider the probability that exactly $\ell_i$ samples lie in the first of this pair (i.e. $j_i$), we would like to show that this probability is similar for a random distribution from $D_{yes}$ and a random distribution from $D_{no}$. In particular, we prove:

\begin{lem}\label{single bin prob lem}
Let $A_i,g_i,m,Q,(j_i,k_i)$ be as in Proposition \ref{main bound prop}, and let $1\leq i \leq s$, and let $B_i$ and $N$ be non-negative integers. Then we have that if $N$ i.i.d. samples are taken from a probability distribution $p$ taken from either $D_{yes}$ or $D_{no}$, and consider this distribution conditioned on exactly $B_i$ samples lying in the $i$th pair of bins. Let $X_i^{yes}$ and $X_i^{no}$ be the distributions on the number of samples drawn from the bin $j_i$ in the case where $p$ is taken from $D_{yes}$ or $D_{no}$ respectively. Then
$$
\dtv(X_i^{yes},X_i^{no}) \leq O(1/s^2) + m^4 O(\sqrt{B_i \log(s)} + x_{max}B_i)(1+x_{max})^{O(\sqrt{B_i \log(s)} + x_{max}B_i)}O(\sqrt{x_{max}^2B_i \log(s)} + x_{max}^2 B_i)^m.
$$
\end{lem}
\begin{proof}
We begin by proving this in the case where $Q_{j_i} = Q_{k_i}$ and will reduce to this case later.

The key observation here is that if we let $x = \frac{\delta_i}{Q_{j_i}}$ then a sample landing in the $i$th pair of bins will have a probability of
$$
\frac{p_{j_i}}{p_{j_i}+p_{k_i}} = \frac{Q_{j_i}+\delta_i}{Q_{j_i}+Q_{k_i}} = \frac{1+x}{2}
$$
of landing in bin $j_i$. Thus, conditioned on $\delta_i$, $X_i^{yes/no}$ is distributed as the binomial distribution $\mathrm{Bin}(B_i,(1+x)/2)$. Therefore, the probabilities of $X_i^{yes}$ and $X_i^{no}$ being equal to $\ell$ are just
$$
\E_{x \sim F^{A_i,g_i,m}_{yes}/Q_{j_i}}[\pr(\mathrm{Bin}(B_i,(1+x)/2) = \ell)] = 2^{-B_i}\binom{B_i}{\ell}\E_{x \sim F^{A_i,g_i,m}_{yes}/Q_{j_i}}\left[(1+x)^\ell (1-x)^{B_i-\ell}\right]
$$
and
$$
2^{-B_i}\binom{B_i}{\ell}\E_{x \sim F^{A_i,g_i,m}_{no}/Q_{j_i}}\left[(1+x)^\ell (1-x)^{B_i-\ell}\right],
$$
respectively.

Our goal will be to show that (at least for $\ell$ close to $B_i/2$, which it will be with high probability) that these are close. The basic plan here is to approximate the term $(1+x)^\ell (1-x)^{B_i-\ell}$ by its Taylor series about $x=0$. We note that since the low order moments of $x\sim F^{A_i,g_i,m}_{yes}/Q_{j_i}$ and $x\sim F^{A_i,g_i,m}_{no}/Q_{j_i}$ are identical, these terms will cancel exactly, leaving only the Taylor error terms to contend with, which we will prove are small.

However, before we do this, we first want to deal with the outer term $2^{-B_i}\binom{B_i}{\ell}$. In particular, we show that it is at most $1$. In fact,
$$\binom{B_i}{\ell} 2^{-B_i}< 2^{-B_i}\sum_{\ell} \binom{B_i}{\ell}1^{\ell}1^{B_i-\ell}=1$$
by the Binomial Theorem.

We next let $f(x) := (1+x)^{\ell} (1-x)^{B_i-\ell}$. By Taylor expanding about $x=0$, we find that
\begin{equation}
\label{eqn1}
f(x) = f(0) + f'(0) x + \frac{f''(0)}{2!} x^2 \dots + \frac{f^{(m-1)}(0)}{(m-1)!}x^{m-1} + R_m(x)
\end{equation}
where $R_m(x) = \frac{f^m (\zeta)}{m!}x^m$ for some $\zeta$ between 0 and $x$.

As $E_{\delta_i \in F_{yes}}[\delta_i^k]=E_{\delta_i \in F_{no}}[\delta_i^k]$ for $k<m$ by Lemma \ref{F moment matching lem}, we have $E_{x \sim F^{A_i,g_i,m}_{yes}/Q_{j_i}} \left[ f(x_i) \right]=E_{x \sim F^{A_i,g_i,m}_{no}/Q_{j_i}} \left[ f(x_i) \right]$ if $B_i<m$. On the other hand, if $B_i \geq m$, the expectations of the non-remainder terms over $x\sim F^{A_i,g_i,m}_{yes}/Q_{j_i}$ and $x\sim  F^{A_i,g_i,m}_{no}/Q_{j_i}$ will be the same. Thus, in that case we have that
\begin{align*}
E_{x \sim F^{A_i,g_i,m}_{yes}/Q_{j_i}}  \left[ f(x) \right] &- E_{x \sim F^{A_i,g_i,m}_{no}/Q_{j_i}} \left[ f(x) \right]
=E_{x \sim F^{A_i,g_i,m}_{yes}/Q_{j_i}} R_m(x) - E_{x \sim F^{A_i,g_i,m}_{yes}/Q_{j_i}} R_m(x).
\end{align*}
We will try to bound this by showing that $|R_m(x)|$ is small at least when $\ell$ is close to $B_i/2$. [[TODO: Mention bounds on $\ell$.]]

\begin{lem}\label{remainder bound lem}
Take $x$ to be a real number with $|x|<\frac{1}{10}$, if $m \geq B_i$ and $\left|\ell - \frac{B_i}{2}\right| \leq \frac{B_i}{5}$, then $$|R_m(x)| \leq (1+|x|)^{|B_i-2\ell|}m^4(2(|x|\sqrt{B_i}+|x||B_i-2\ell| + |x|^2B_i))^m$$ where $R_m(x)$ is given in equation \eqref{eqn1}.
\end{lem}
\begin{proof}
By definition $R_m(x) = x^m/m! f^{(m)}(y)$ for some $y$ between $0$ and $x$. In particular, note that this implies $|y|<1/10$.

Using Leibnitz Rule, the $m$-th derivative of $f(y)$ can be expressed as
\begin{align}
f^{(m)}(y) &= \sum_{t=0}^{m} \binom{m}{t}(B_i-\ell)_{m-t} (\ell)_t (1+y)^{\ell-t} (1-y)^{B_i-\ell-m+t} (-1)^{m-t} \notag \\
\label{equ3}
& = (1+y)^{\ell} (1-y)^{B_i-\ell} \left(\sum_{t=0}^{m} \binom{m}{t} (B_i-\ell)_{m-t} (\ell)_t (\frac{1}{1+y})^t (\frac{1}{1-y})^{m-t} (-1)^{m-t} \right).
\end{align}
Note that the summand in \eqref{equ3} is roughly
$$
\binom{m}{t} (B_i-\ell)^{m-t} \ell^t (\ell)_t (\frac{1}{1+y})^t (\frac{1}{1-y})^{m-t} (-1)^{m-t}.
$$
If this were exactly true, we could use the binomial theorem to rewrite it as
$$
\left( \frac{\ell}{1+y} - \frac{B_i-\ell}{1-y} \right)^m,
$$
allowing us to take advantage of significant cancellation of terms. Unfortunately, the falling factorials $(B_i-\ell)_{m-t}$ and $ (\ell)_t$ are not exactly equal to the relevant exponentials. However, we can make use of Stirling numbers to write them in terms of similar exponentials.
\begin{defn}
The unsigned Stirling number of the first kind, usually written as $c(n,k)$ or $\stirling{n}{k}$, is defined to be the number of permutations of $[n]$ with exactly $k$ cycles. The signed Stirling number of the first kind is defined by $s(n,k)=(-1)^{n-k} c(n,k)$.
\end{defn}
In particular, we make use of the fact:
\begin{fact}
For any non-negative integer $n$ and real number $z$ we have that
$$(z)_n = \sum_{k}s(n,k)z^k.$$
\end{fact}
Using this, we may rewrite \eqref{equ3} as
\begin{align*}
&(1+y)^{\ell} (1-y)^{B_i-\ell}\cdot \\
& \sum_{t=0}^{m} \binom{m}{t} \left(\sum_{h'} (B_i-\ell)^{m-t-h'} (-1)^{h'} \stirling{m-t}{m-t-h'}\right) \left(\sum_{h} (\ell)^{h-s} (-1)^{h} \stirling{t}{t-h}\right) \left(\frac{1}{1+y}\right)^t \left(\frac{-1}{1-y}\right)^{m-t}
\end{align*}
Interchanging the order of summations yields
\begin{align}
\label{equ2}
&(1+y)^{\ell} (1-y)^{B_i-\ell}\cdot \\
& \sum_{h,h'} \left(\sum_{t=0}^{m} \binom{m}{t} \stirling{t}{t-h} \stirling{m-t}{m-t-h'}(-1)^{h+h'} (\ell)^{t-h} (B_i-\ell)^{m-t-h'} \left(\frac{1}{1+y}\right)^t \left(\frac{-1}{1-y}\right)^{m-t}\right).\notag
\end{align}
To make further progress, we would like to simplify $\binom{m}{t}\stirling{t}{t-h} \stirling{m-t}{m-t-h'}$. Specifically, we use a lemma about Stirling numbers to find an alternative expression of $\stirling{t}{t-h}$ and $ \stirling{m-t}{m-t-h'}$ and get cancellation of the binomial coefficients.
\begin{lem}\label{stirling equation lem}
There exist some constants $0 \leq c_{f, h} \leq 1$ such that for all $t,h$, $\stirling{t}{t-h} = \sum_{f=h+1}^{2h} (t)_f c_{f, h}$
\end{lem}
\begin{proof}
We analyze these Stirling numbers using a combinatorial approach. $\stirling{t}{t-h}$ is the number of permutations of $[t]$ with exactly $t-h$ cycles. Such permutations should have between $h+1$ and $2h$ non fixed points. Let $f$ be the number of non-fixed points in this permutation. Then $\stirling{t}{t-h}$ can be represented as $\sum_{f=h+1}^{2h} \binom{t}{f} T_{f, f-h}$ where $T_{f, f-h}$ represents the number of permutations of $f$ elements with no fixed points that have exactly $f-h$ cycles. Note that $T_{f, f-h} \leq f!$.

We have $\stirling{t}{t-h} = \sum_{f=h+1}^{2h} \frac{t!}{f!(t-f)!} T_{f,f-h}=\sum_{f=h+1}^{2h} (t)_f \frac{T_{f, f-h}}{f!}.$ So taking $c_{f, h} = \frac{T_{f, f-h}}{f!}$, we are done.
\end{proof}
Substituting the result of Lemma \ref{stirling equation lem} into Stirling numbers showing up in \eqref{equ2}, we get that
$$\binom{m}{t}\stirling{t}{t-h} \stirling{m-t}{m-t-h'}=\frac{m !}{t!(m-t)!}\sum_{f=h+1}^{2h} \frac{t!}{(t-f)!} c_{f,h} \sum\limits_{g=h'+1}^{2h'} \frac{(m-t)!}{(m-t-g)!} c_{g,h'}$$

$$=\sum_{f=h+1}^{2h} \sum_{g=h'+1}^{2h'} \frac{(m-g-f)!}{(t-f)!(m-t-g)!}(m)_{f+g}c_{f,h}c_{g,h'}=\sum_{f=h+1}^{2h} \sum_{g=h'+1}^{2h'}\binom{m-g-f}{t-f}(m)_{f+g}c_{f,h}c_{g,h'}.$$
Substituting this result into equation \eqref{equ2}, we have $f^{(m)}(y)$ equals
\begin{align*}
& (1+y)^{\ell} (1-y)^{B_i-\ell}\cdot\\ & \sum_{t=0}^{m}\sum_{h,h'} \sum_{f=h+1}^{2h} \sum_{g=h'+1}^{2h'}\binom{m-g-f}{t-f} (-1)^{h+h'+g} \left(\frac{\ell}{1+y}\right)^{t-f} \left(-\frac{B_i-\ell}{1-y}\right)^{m-t-g} \frac{\ell^{f-h}(B_i-\ell)^{g-h'}}{(1+y)^f(1-y)^g}(m)_{f+g}c_{f,h}c_{g,h'}.
\end{align*}
Applying the binomial theorem to the sum $\sum_{t=0}^{m}\binom{m-g-f}{t-f} \left(\frac{\ell}{1+y}\right)^{t-f} \left(-\frac{B_i-\ell}{1-y}\right)^{m-t-g}$, we can get that the above is equal to
$$
(1+y)^{\ell} (1-y)^{B_i-\ell} \cdot
\sum_{h,h'} \sum_{f=h+1}^{2h} \sum_{g=h'+1}^{2h'}(-1)^{h+h'+g}\left(\frac{\ell}{1+y} - \frac{B_i-\ell}{1-y}\right)^{m-g-f}\frac{\ell^{f-h}(B_i-\ell)^{g-h'}}{(1+y)^f(1-y)^g}(m)_{f+g}c_{f,h}c_{g,h'}.
$$
Given $0<c_{f,h},c_{g,h'}<1$, we have that $|f^{(m)}(y)|$ is at most
\begin{equation*}
(1+y)^{\ell} (1-y)^{B_i-\ell} \sum_{h,h'}\sum_{f=h+1}^{2h} \sum_{g=h'+1}^{2h'}\left|\frac{\ell}{1+y} - \frac{B_i-\ell}{1-y}\right|^{m-g-f}\frac{\ell^{f-h}(B_i-\ell)^{g-h'}}{(1+y)^f(1-y)^g}(m)_{f+g}.
\end{equation*}
Note that $(m)_{f+g}$ is at most $m!$, but vanishes if $f+g>m$. In order for this not to happen, it must be the case that $h+h' < m$. This means that there are at most $m^4$ non-vanishing terms in the above sum as each of $h$ and $h'$ can take at most $m$ values and for each pair of values, there are at most $m$ possibilities for each of $f$ and $g$. Therefore, we have that the above is at most
\begin{equation}\label{find exponents equation}
[(1+y)^{\ell} (1-y)^{B_i-\ell}] m^4 m! \max_{\substack{2h \geq f > h \geq 0\\ 2h' \geq g > h' \geq 0\\ f+g \leq m}} \left[\left|\frac{\ell}{1+y} - \frac{B_i-\ell}{1-y}\right|^{m-g-f}\frac{\ell^{f-h}(B_i-\ell)^{g-h'}}{(1+y)^f(1-y)^g}\right].
\end{equation}
In order to bound \eqref{find exponents equation}, we want to find the largest summands. For this we note that increasing $f$ or $g$ decreases the exponent of $\left|\frac{\ell}{1+y} - \frac{B_i-\ell}{1-y}\right|$ while increasing $f$ increases the exponent of $\left(\frac{\ell}{1+y}\right)$ and increasing $g$ increases the exponent of $\left(\frac{B_i-\ell}{1-y}\right)$. To make progress, we need to understand the relative sizes of these terms:
\begin{claim}
We have that $|\frac{\ell}{1+y} - \frac{B_i-\ell}{1-y}| \leq \frac{\ell}{1+y}$ and $|\frac{\ell}{1+y} - \frac{B_i-\ell}{1-y}| \leq \frac{B_i-\ell}{1-y}$.
\end{claim}
\begin{proof}
It suffices to show that $\frac{\ell}{1+y}$ and $\frac{B_i-\ell}{1-y}$ are within a factor of two of each other. For this, we note that as $|y| < 1/10$ that the ratio of $1+y$ to $1-y$ is between $9/11$ and $11/9$. Furthermore, as $|B_i-2\ell| < B_i/5$, we have that
the ratio of $B_i-\ell$ to $\ell$ is the same as the ratio of $(\ell/B_i) + (B_i-2\ell)/B_i$ to $(\ell/B_i)$, which is between $4/5$ and $6/5$. Multiplying these together yields our result.
\end{proof}
Applying this, we find that the maximum in \eqref{find exponents equation} is attained either when $f+g=m$ or when $f=2h$ and $g=2h'$. In the former case we have

\begin{align*}
& \max_{\substack{2h \geq f > h \geq 0\\ 2h' \geq g > h' \geq 0\\ f+g = m}} \frac{\ell^{f-h}(B_i-\ell)^{g-h'}}{(1+y)^f(1-y)^g}\\
= & \max_{\substack{2h \geq f > h \geq 0\\ 2h' \geq g > h' \geq 0\\ f+g = m}} \left(\frac{\ell^{f-h}}{(1+y)^f} \right) \left(\frac{(B_i-\ell)^{g-h'}}{(1-y)^g} \right)\\
\leq & \max_{f+g = m}\left(\frac{\ell^{f/2}}{(1+y)^f} \right) \left(\frac{(B_i-\ell)^{g/2}}{(1-y)^g} \right)\\
= & \max\left(\frac{\sqrt{\ell}}{1+y}, \frac{\sqrt{B_i-\ell}}{1-y} \right)^m.
\end{align*}
In the latter case, it gives
\begin{align*}
& \max_{\substack{2h+2h' \leq m}} \left|\frac{\ell}{1+y} - \frac{B_i-\ell}{1-y}\right|^{m-2h-2h'}\frac{\ell^{h}(B_i-\ell)^{h'}}{(1+y)^{2h}(1-y)^{2h'}}\\
\leq & \max\left( \left|\frac{\ell}{1+y} - \frac{B_i-\ell}{1-y}\right|,\frac{\sqrt{\ell}}{1+y}, \frac{\sqrt{B_i-\ell}}{1-y} \right)^m.
\end{align*}

Thus, in either case we have that $|f^{(m)}(x)|$ is at most
$$
[(1+y)^{\ell} (1-y)^{B_i-\ell}] m^4 m! \max\left( \left|\frac{\ell}{1+y} - \frac{B_i-\ell}{1-y}\right|,\frac{\sqrt{\ell}}{1+y}, \frac{\sqrt{B_i-\ell}}{1-y} \right)^m.
$$
To bound the maximum, we note that
$$
\frac{\sqrt{\ell}}{1+y}, \frac{\sqrt{B_i-\ell}}{1-y} \leq \frac{\sqrt{B_i}}{9/10} \leq 2\sqrt{B_i}.
$$
On the other hand,
\begin{align*}
\left|\frac{\ell}{1+y} - \frac{B_i-\ell}{1-y}\right| & = \frac{1}{1-y^2}\left|(B_i-2\ell)-yB_i \right|\\
& \leq 2(|B_i-2\ell| + |x|B_i).
\end{align*}

Finally, note that $(1+y)(1-y) = 1-y^2 < 1.$ Therefore, we have that
$$
[(1+y)^{\ell} (1-y)^{B_i-\ell}] \leq \max((1+y),(1-y))^{|(B_i-\ell)-\ell|} = (1+|x|)^{|B_i-2\ell|}.
$$

Putting this together we have that for $|x|<1/10$ that
$$
R_m(x) = |x^m f^{(m)}(y) / m!| \leq (1+|x|)^{|B_i-2\ell|}m^4(2(|x|\sqrt{B_i}+|x||B_i-2\ell| + |x|^2B_i))^m.
$$
As desired.

\end{proof}

So for any value of $\ell$ we have that
\begin{align*}
& \left|\pr(X_i^{yes} = \ell) - \pr(X_i^{no} = \ell) \right|\\
& \leq \left|\E_{x\sim F^{A_i,g_i,m}_{yes}/Q_{j_i}}[f(x)] - \E_{x\sim F^{A_i,g_i,m}_{no}/Q_{j_i}}[f(x)] \right|\\
& = \left|\E_{x\sim F^{A_i,g_i,m}_{yes}/Q_{j_i}}[a_0 + a_1 x +\ldots +a_{m-1}x^{m-1} + R_m(x)] - \E_{x\sim F^{A_i,g_i,m}_{no}/Q_{j_i}}[a_0 + a_1 x +\ldots +a_{m-1}x^{m-1} + R_m(x)] \right|\\
& = \left|\sum_{k=0}^{m-1}\left(\E_{x\sim F^{A_i,g_i,m}_{yes}/Q_{j_i}}[a_kx^k]- \E_{x\sim F^{A_i,g_i,m}_{no}/Q_{j_i}}[a_kx^k]\right) + \E_{x\sim F^{A_i,g_i,m}_{yes}/Q_{j_i}}[R_m(x)]- \E_{x\sim F^{A_i,g_i,m}_{no}/Q_{j_i}}[R_m(x)] \right|\\
& = \left|\E_{x\sim F^{A_i,g_i,m}_{yes}/Q_{j_i}}[R_m(x)]- \E_{x\sim F^{A_i,g_i,m}_{no}/Q_{j_i}}[R_m(x)] \right|.
\end{align*}
Since for all $x$ in either distribution, we have $|x|<1/10$, this is at most
$$
2(1+|x|)^{|B_i-2\ell|}m^4(2(|x|\sqrt{B_i}+|x||B_i-2\ell| + |x|^2B_i))^m
$$
by Lemma \ref{remainder bound lem}.

While this bound is fairly good when $\ell$ is close to $B_i/2$, it is less useful when they are far apart. Fortunately, since $B_i \geq m > C \log(s)$ we have by a Chernoff bound that except for with probability $1/s^2$ that $\mathrm{Bin}(B_i,(1+x)/2)$ is within $O(\sqrt{B_i \log(s)})$ of $B_i(1+x)/2$. Therefore, for a sufficiently large constant $A$,
$$
\pr(|2X_i^{yes}-B_i| > x_{max}B_i + A \sqrt{B_i \log(s)} ) < 1/s^2
$$
and similarly for $X_i^{no}$. Therefore, we have that $\dtv(X_i^{yes},X_i^{no})$ is at most
$$
O(1/s^2)+ \sum_{\ell: |B_i-2\ell| < x_{max}B_i + A \sqrt{B_i \log(s)}}\left|\pr(X_i^{yes} = \ell) - \pr(X_i^{no} = \ell) \right|.
$$
Using the above to bound the differences in probabilities, we get a final bound of:
$$
O(1/s^2) + m^4 O(\sqrt{B_i \log(s)} + x_{max}B_i)(1+x_{max})^{O(\sqrt{B_i \log(s)} + x_{max}B_i)}O(\sqrt{x_{max}^2B_i \log(s)} + x_{max}^2 B_i)^m.
$$

This completes our proof when $Q_{j_i} = Q_{k_i}$. In general, we can assume without loss of generality that $Q_{k_i} \geq Q_{j_i}$. We will then sub-divide the bin $k_i$ into two sub-bins with probability masses $Q_{k_i}-Q_{j_i}$ and $Q_{j_i}-\delta_i$. We can think of taking a sample from $p$ conditioned on lying in $\{j_i,k_i\}$ as first with probability $(Q_{k_i}-Q_{j_i})/(Q_{k_i}+Q_{j_i})$ taking a sample from the first sub-bin (and thus landing in bin $k_i$), and otherwise taking a sample from $j_i$ or $k_i$ with probabilities $(Q_{j_i} \pm \delta_i)/(2Q_{j_i})$. If we are taking $B_i$ samples from this pair, we can imagine this as first taking $X \sim \mathrm{Bin}(B_i,(Q_{k_i}-Q_{j_i})/(Q_{k_i}+Q_{j_i}))$ samples from this extra sub-bin and then taking $B_i' = B_i - X$ samples from the remaining pair. However, we note that the distribution of samples obtained in the first bin of this pair is distributed exactly as it would have been if $B_i'$ samples were originally taken conditioned on lying in a pair of bins with probabilities $Q_{j_i} \pm \delta$. As this situation has already been analyzed (in the case where $Q_{j_i} = Q_{k_i}$), we know that the resulting total variational distance is at most
$$
O(1/s^2) + m^4 O(\sqrt{B_i' \log(s)} + x_{max}B_i')(1+x_{max})^{O(\sqrt{B_i' \log(s)} + x_{max}B_i')}O(\sqrt{x_{max}^2B_i' \log(s)} + x_{max}^2 B_i')^m.
$$
Taking the expectation of this over $B_i'$ (which is always less than $B_i$) yields our full result.
\end{proof}

Lemma \ref{single bin prob lem} provides fairly good bounds so long as $B_i$ is not too large. However, the total number of samples $N$ that we are taking might be substantially larger. Fortunately, we can say that with high probability that no pair of bins contains too many samples. In particular we show:
\begin{lem}\label{Bi bound lem}
If $N>\frac{6\log s}{\min_{i} (Q_{j_i}+Q_{k_i})}$ and $B = 2\max_{i}(Q_{j_i}+Q_{k_i}) N$, then if $N$ i.i.d. samples are drawn from a distribution $p$ taken from either $D_{yes}$ or $D_{no}$, then with probability at least $1-1/s$ there is no pair of bins $(j_i,k_i)$ receiving a total of more than $B$ of these samples.
\end{lem}
\begin{proof}
By construction, for $D_{yes/no}^N$, $\mu_i := \E \left[ B_i \right] = (Q_{j_i}+Q_{k_i}) N > 6\log s$. Note that $B_i$ is a sum of independent and identically distributed indicator random variables. Applying the Chernoff Bounds,  $\pr(B_i \geq (1+\delta)\mu_i) \leq e^{-\frac{\delta^2 \mu_i}{3}}$. Letting $\delta = 1$, then we have that $\pr(B_i \geq 2\mu_i ) \leq \frac{1}{s^2}$. As $B \geq 2\mu_i$, this says that with probability at least $1-1/s^2$ that the $i$th pair of bins does not contain more than $B$ samples. Our result now follows by taking a union bound over $i$.
\end{proof}

\subsection{Proof of Proposition \ref{main bound prop}}

We are now ready to prove the full Proposition \ref{main bound prop}.

\begin{proof}
Let $B_i$ be the number of samples from the $i$th pair of bins, ${j_i}$ and ${k_i}$. Define $U$ to be the vector of values $(B_1, \cdots , B_s)$ as well as the number of samples in each unpaired bin. By our construction, the distribution of $U$ is the same for any $p$ taken from $D_{yes}$ or $D_{no}$, independently of $\delta_i$s, as $p_{j_i}+p_{k_i}$ is always $Q_{j_i}+Q_{k_i}$. So
\begin{equation}
\label{equ4}
d_{TV}(D_{yes}^N, D_{no}^N) = E_U[ d_{TV}(D_{yes}^N | U, D_{no}^N | U) ] \leq \pr_U(\exists i:B_i \geq B)+\max_{U:B_i<B\textrm{ for all }i} d_{TV}(D_{yes}^N | U, D_{no}^N | U).
\end{equation}
For the first term, we note that Lemma \ref{Bi bound lem} implies that the probability that some $B_i$ is more than $B$ is at most $1/s$. To deal with the second term, we note that after conditioning on $U$, either $D_{yes}^N$ or $D_{no}^N$, we observe that the choice of $\delta_i$s are independent for each pair of bins. This implied that conditioned on $U$, the number of samples drawn from each of $j_i$ and $k_i$ are independent for each $i$. As the distributions $D_{yes}^N$ and $D_{no}^N$ are symmetric in the sense that seeing a collection of samples in some order is as likely as seeing those samples in any other order, the total variational distance between these distributions conditioned on $U$ is the same as the variational distance between their distributions over the counts of numbers of samples from each bin. These distributions in turn are product distributions over pairs of bins, and thus we can bound \eqref{equ4} by
\begin{equation}
\label{equ5}
\max_{U:B_i \leq B} \sum_{i=1}^{s} d_{TV}(X_{i,yes}^{B_i},X_{i,no}^{B_i})+\frac{1}{s},
\end{equation}
where $X_{i,yes}^{B_i}$ is the distribution over the number of samples a random $p$ from $D_{yes}$ draws from $j_i$ conditioned on the fact that it drew a total of $B_i$ samples from $\{j_i,k_i\}$, and $X_{i,no}^{B_i}$ is defined similarly. However, by Lemma \ref{single bin prob lem} and the fact that $B_i \leq B$, the $i$th term in this sum is at most
$$
O(1/s^2) + m^4 O(\sqrt{B \log(s)} + x_{max}B)(1+x_{max})^{O(\sqrt{B \log(s)} + x_{max}B)}O(\sqrt{x_{max}^2B \log(s)} + x_{max}^2 B)^m.
$$
Summing over all $i$ from $1$ to $s$ and adding in the extra $1/s$ term gives our final bound of
$$
O(1/s) + m^4 s O(\sqrt{B \log(s)} + x_{max}B)(1+x_{max})^{O(\sqrt{B \log(s)} + x_{max}B)}O(\sqrt{x_{max}^2B \log(s)} + x_{max}^2 B)^m.
$$
\end{proof}

\section{One-Dimensional Monotonicity Testing}\label{one dim monotonicity sec}

As a warmup we will prove the $d=1$ version of Theorem \ref{monotonicity thm}.

\subsection{Construction}

In this section, we focus on getting a new lower bound of testing monotone distribution over $[n]$. We will do this by producing a version of the construction in Section \ref{generic construction section} so that a distribution from $D_{yes}$ is always monotone and a distribution from $D_{no}$ is far from monotone with high probability. We begin by proving it for $n$ not too large.

In particular, assume that $n$ is an even number with $C^2 \log(1/\eps) < n < \frac{1}{C^4(\log\frac{1}{\epsilon})^3\epsilon}$ for some sufficiently large constant $C$ and assume that $\eps$ is sufficiently small. We begin with defining a base distribution $Q$ over $S = [n]$ by
$$Q_{2i-1}=Q_{2i} = \frac{5}{4n}+\frac{1}{2n^2}-\frac{i}{n^2}, 1 \leq i \leq \frac{n}{2}.$$

Note that $Q_i \geq Q_j$ for any $i \leq j$ and $\sum_{i=1}^{n} Q_i =1$. We define the sequence of pairs $(j_i,k_i)_{1 \leq i \leq \frac{n}{2}}$ by $j_i = 2i-1, k_i = 2i$. Note that these cover all bins in $S$ exactly once.

To complete the construction, we need to define values of $A_i, g_i,$ and $m$. In particular, we let $m$ be the smallest odd integer that is more than $C\log(1/\eps)$.

Note that we have $nm^3 < 1/(C \eps)$, and thus $\frac{1}{4n^2} > \frac{8m^3 \epsilon}{n}$. Therefore, taking $A_i = A = \frac{8m^3 \epsilon}{n}$ for all $i$, we have $\frac{1}{4n^2} > A$. We also let $g_i=\cos(\frac{\pi}{m})$ for all $i$. Since $A_i$ and $g_i$ are both constants, the distributions of $F_{yes/no}^i$ are identical for all $i$. For convenience of notation, we denote this distribution to be $F_{yes/no}$. Given $\epsilon>0$ sufficiently small, we can assume that $n$ and $m$ are larger than sufficiently large constants. It's easy to check that this construction satisfies $A<\frac{\min_i Q_i}{2+2|g|}$.

In order to prove the monotonicity/non-monotonicity of $D_{yes}/D_{no}$ we will need some properties of the $F_{yes/no}$ with these particular parameters. In fact, we will prove a slightly more general form:
\begin{lem}\label{F bounds lem}
For $m$ a sufficiently large positive odd integer, $g = \cos(\frac{\pi}{m})$, and $A>0$, $F^{A,g,m}_{yes}$ and $F^{A,g,m}_{no}$ have the following properties:
\begin{enumerate}
\item If $\delta$ is taken from either distribution $|\delta| < 2A$ almost surely.
\item If $\delta$ is taken from $F^{A,g,m}_{yes}$, then $\delta \geq 0$ almost surely.
\item If $\delta$ is taken from $F^{A,g,m}_{no}$, then there is a probability of $1/m$ that $\delta$ is negative, in which case we have $\delta < -A/m^2.$
\end{enumerate}
\end{lem}
\begin{proof}
Property 1 follows from the fact that the cosine terms all have absolute value at most 1. For property 2, since $m$ is odd, we have that $\cos(\frac{2\pi a}{m}) \geq \cos(\frac{(m-1)\pi}{m})=-\cos(\frac{\pi}{m})$, so all $\delta$s drawn from $F_{yes}$ will be non-negative. For distribution of $F_{no}$, we have $A(\cos(\frac{2\pi (a+\frac{1}{2})}{m}) + \cos(\frac{\pi}{m}))<0$ if and only if $a=\frac{m-1}{2}$. Since there are $m$ choices of $a$, we conclude that $\delta$ drawn from $F_{no}$ will have $\frac{1}{m}$ chance to be negative, and the negative value is $A(\cos(\frac{\pi}{m}) - 1)$. By Taylor expanding $\cos(x)$ about 0, we find that it is at most $-A(\frac{\pi^2}{2m^2}-\frac{\pi^4}{24m^4})$. Given that $m$ is sufficiently large this is at most $-A/m^2.$
\end{proof}

Lemma \ref{F bounds lem} ensures that $\delta$s drawn from either distribution are small. In addition, it guarantees that $\delta$s drawn from $F_{yes}$ are positive with probability 1 and $\delta$s drawn from $F_{no}$ are negative with probability $\frac{1}{m}$. This makes sure that the distributions in $D_{yes}$ and the distributions in $D_{no}$ are different in terms of being monotone or not, as we can see in the later analysis.

We want to show that a random distribution drawn from $D_{yes}$ is monotone, and a random distribution drawn from $D_{no}$ is some distance from monotone with high probability. This will mean that any monotonicity tester will be able to distinguish between a distribution from $D_{yes}$ and a distribution from $D_{no}$, which is impossible without many samples by Proposition \ref{main bound prop}.
\begin{lem}
A distribution $p$ drawn from $D_{yes}$ is monotone with probability 1.
\end{lem}
\begin{proof}
We will prove that $p_i \geq p_{i+1}$ for all $i$. Firstly, we note that for odd $i$, $Q_i=Q_{i+1}$. According to construction, $p_i=Q_i+\delta_{\frac{i+1}{2}}$ and $p_{i+1}=Q_i-\delta_{\frac{i+1}{2}}$. Applying Lemma \ref{F bounds lem}, for $\delta \in F_{yes}$, $\delta \geq 0$ gives $p_i \geq p_{i+1}$.

For even $i$, $p_i=\frac{5}{4n}+\frac{1}{2n^2}-\frac{i}{2n^2}-\delta_{\frac{i}{2}}$ and $p_{i+1}=\frac{5}{4n}+\frac{1}{2n^2}-\frac{i+2}{2n^2}+\delta_{\frac{i+2}{2}}$. Thus,
\begin{align*}
p_i-p_{i+1} & = \frac{1}{n^2} - (\delta_{\frac{i}{2}}+\delta_{\frac{i+2}{2}})
\end{align*}
By Lemma \ref{F bounds lem}, $|\delta_j| < 2A < 1/(2n^2)$ for each $j$. Thus
$$
p_i - p_{i+1} \geq \frac{1}{n^2} - \frac{2}{2n^2} = 0.
$$
This completes our proof.
\end{proof}

In contrast to distributions in $D_{yes}$, the distributions in $D_{no}$ is at least $\epsilon$ far from monotone with high probability. In order to show this, we first need a lemma allowing us to show that some distributions $p$ are far from \emph{any} monotone distribution.

\begin{lem}\label{monotone distance lem}
For a distribution $p$, and $q$ an arbitrary monotone distribution, then $\dtv(p,q) \geq \sum_{i=1}^{\frac{n}{2}} \gamma_i$, where
$$\gamma_i =
\begin{cases}
|p_{2i-1}-p_{2i}|/2 & \textrm{if } p_{2i-1}-p_{2i} < 0 \\
0 & \textrm{if }  p_{2i-1}-p_{2i} \geq 0.
\end{cases}$$
\end{lem}
\begin{proof}
We will show that $(|q_{2i-1}-p_{2i-1}|+|q_{2i}-p_{2i}|)/2 \geq \gamma_i$. In particular, if $p_{2i-1}-p_{2i} \geq 0$, $\gamma_i=0$ and we have our desired inequality. If $p_{2i-1}-p_{2i} < 0$, then $\gamma_i = (p_{2i}-p_{2i-1})/2>0$. By definition, $q$ monotone implies $q_{2i-1} \geq q_{2i}$. This means that
$$|q_{2i-1}-p_{2i-1}|+|q_{2i}-p_{2i}| \geq (q_{2i-1}-p_{2i-1})+(p_{2i}-q_{2i}) = q_{2i-1}-q_{2i}+2\gamma_i \geq 2\gamma_i.$$

Summing this inequality over all $i$, we have:
$$\dtv(p,q) = \frac{1}{2}\sum_{i=1}^{n} |q_i-p_i| = \sum_{i=1}^{\frac{n}{2}}((|q_{2i-1}-p_{2i-1}|+|q_{2i}-p_{2i}|))/2 \geq \sum_{i=1}^{\frac{n}{2}} \gamma_i.$$
This completes our proof.
\end{proof}
We can now use Lemma \ref{monotone distance lem} to show that a distribution from $D_{no}$ is far from monotone with high probability.
\begin{lem}
With $99\%$ probability, a random distribution drawn from $D_{no}$ is $\epsilon$ far from monotone.
\end{lem}
\begin{proof}
Let $\gamma_i$ be as in Lemma \ref{monotone distance lem}, we have that
$$\gamma_i =
\begin{cases}
|p_{2i-1}-p_{2i}|/2 & \textrm{if } p_{2i-1}-p_{2i} < 0 \\
0 & \textrm{if }  p_{2i-1}-p_{2i} \geq 0.
\end{cases}
$$
Note that $p_{2i-1}-p_{2i} = (Q_{2i-1} + \delta_i ) - (Q_{2i} -\delta_i) = 2\delta_i$. Thus we have that
$$\gamma_i =
\begin{cases}
-\delta_i & \textrm{if } \delta_i < 0 \\
0 & \textrm{if }  \delta_i \geq 0.
\end{cases}
$$

By Lemma \ref{F bounds lem}, we have that each $\gamma_i$ is positive (with absolute value at least $8\eps m/n$) independently with probability $1/m$. Let $X$ be the number of these positive terms. We have that $X\sim \mathrm{Bin}(n/2,1/m)$. As $n/m$ is at least a large constant, we have with $99\%$ probability that $X > n/(4m)$. If this holds then by Lemma \ref{monotone distance lem} the distance of $p$ from the nearest monotone distribution is at least
$$
\sum_{i=1}^{n/2} \gamma_i \geq (n/(4m) ) ( 8\eps m/n ) > \eps.
$$
This completes our proof.
\end{proof}

We are now prepared to prove Theorem \ref{monotonicity thm} when $d=1$ and $n < \frac{1}{C^4(\log\frac{1}{\epsilon})^3\epsilon}$ is even. Let $N$ be a sufficiently small multiple of $n/(m^6 \log(n) \eps^2)$ and suppose for sake of contradiction that there is a monotonicity algorithm with a probability $2/3$ of success using only $N$ samples. Running this algorithm should be able to distinguish $N$ samples taken from a random distribution from $D_{yes}$ and a random distribution from $D_{no}$ with probability of success at least $3/5$ since the distribution in the former case will be monotone, and the distribution in the latter will be at least $\eps$-far from monotone with probability at least $99\%$.

On the other hand, we can apply Corollary \ref{main bound cor} here as $B = \Theta(N/n)$ and $x_{max} = O(A n) = O(m^3 \eps)$. Thus, $B x_{max}^2 = O(N m^3 \eps^2/n),$ which is at most a small multiple of $1/\log(s)$. This implies that $\dtv(D_{yes}^N, D_{no}^N) < 1/100$, and thus the difference in the probability that our tester accepts a distribution from $D_{yes}$ given $N$ samples can differ from the probability of accepting a distribution from $D_{no}$ given $N$ samples by at most $1/100$.

This completes the proof when $n$ is even and at most $\frac{1}{C^4(\log\frac{1}{\epsilon})^3\epsilon}$. For other $n$, we let $n_0$ be the largest even number less than both $n$ and $\frac{1}{C^4(\log\frac{1}{\epsilon})^3\epsilon}$. We note that a monotonicity tester on $[n]$ can be used to obtain a monotonicity tester on $[n_0]$ simply by ignoring the extra bins in the domain. Thus, we get a lower bound of $\Omega(n_0/(\log^7(1/\eps) \eps^2)) = \Omega( \min(n,(1/\eps)/ \log^3(1/\eps)) / ( \eps \log^7(1/\eps))).$

This completes our proof.

\section{Multidimensional Monotonicity Testing}\label{multidim monotonicity sec}
In this section, we generalize the results of the previous section to cover $d$-dimensional monotonicity testing.

\subsection{Construction}
For the one dimensional case we were able to modify our monotone base distribution $Q$ to make it non-monotone by exchanging bits of probability mass between adjacent bins. In the high dimensional case however, it is not clear what the appropriate generalization of this should be, especially given that there are pairs of bins $\textbf{i}$ and $\textbf{j}$ that are incomparable to each other in the relevant ordering. Thus, in order to construct $D_{yes/no}$ for the multidimensional case, we have to find comparable pairs of bins to move. Here we introduce the notion of cubes and halfcubes for a distribution over $[n]^d$ so that most bins in a halfcube are comparable to a bin in another halfcube within the same cube. Once again, we start by proving it when $n$ is not too large.

Suppose $Cd\log(1/\eps) < n < \frac{d}{(C^2\log\frac{1}{\epsilon})^3\epsilon}$, and $d<(C^2\log\frac{1}{\epsilon})^3$ for some sufficiently large constant $C$ and that $2d|n$. We begin by separating a distribution over $[n]^d$ into $(\frac{n}{2d})^d$ cubes with each of them having $(2d)^d$ bins, with the idea of using these cubes as a unit to replace the pairs of bins in the one dimensional construction. More formally,
\begin{defn}
Let $\textbf{1}$ be the d-dimensional vector $(1,1,1...,1)$. For $\textbf{i},\textbf{j} \in [n]^d$, we denote $\textbf{i}<\textbf{j}$ when $\textbf{i}_a<\textbf{j}_a$ for all $1 \leq a \leq d$.

For a distribution over $[n]^d$, we define its $\textbf{i}$th cube (for some $\textbf{i} \in [\frac{n}{2d}]^d$) to be the set of $\textbf{j}$th bins where $2d(\textbf{i}-\textbf{1}) < \textbf{j} < 2d\textbf{i}+\textbf{1}$. Within the $\textbf{i}$th cube, we define $J_{\textbf{i}}$ to be the set of bins $\{\textbf{j}:\textbf{j}_1 \leq 2d\textbf{i}_1-d \}$ and $K_{\textbf{i}} = \{\textbf{j}:\textbf{j}_1 > 2d\textbf{i}_1-d \}$. We call $J_{\textbf{i}}$ the first halfcube of the $\textbf{i}$th cube and $K_{\textbf{i}}$ the second halfcube.
\end{defn}
Note that we are separating the $\textbf{i}$th cube into 2 halfcubes based on the magnitude of its first coordinate, so $|J_{\textbf{i}}|=|K_{\textbf{i}}|$ and the $\textbf{i}$th cube is $J_{\textbf{i}} \cup K_{\textbf{i}}$. Our construction will produce distributions that are uniform over each halfcube, so we can construct our base distribution over these halfcubes. In particular, we will use an instantiation of the ensembles from Section \ref{generic construction section} to produce these distributions over halfcubes.

Let the distribution $Q$ over $[\frac{n}{2d}]^d \times \{F,S\}$ given by
$$Q_{\textbf{i},F}=Q_{\textbf{i},S}:=\frac{5(2d)^d}{8n^d} + \frac{2^dd^{d+1}}{4n^{d+1}} - \frac{(\textbf{i}_1+\textbf{i}_2+...\textbf{i}_d)2^{d-1}d^d}{n^{d+1}}, \textbf{i}=(\textbf{i}_1,\textbf{i}_2,...\textbf{i}_d) \in [\frac{n}{2d}]^d.$$
Summing over all $\textbf{i}_1$s and multiply by $d$, we have
$$\sum\limits_{\textbf{i} \in [\frac{n}{2d}]^d }(\textbf{i}_1+\textbf{i}_2+...\textbf{i}_d)=(1+\frac{n}{2d})\frac{n}{4d}(\frac{n}{2d})^{d-1}d=(\frac{d}{2}+\frac{n}{4})(\frac{n}{2d})^d.$$
Therefore, we get
$$\sum\limits_{\textbf{i} \in [\frac{n}{2d}]^d }Q_{\textbf{i},F}+\sum\limits_{\textbf{i} \in [\frac{n}{2d}]^d }Q_{\textbf{i},S}=2(\frac{5}{8}+\frac{d}{4n}-(\frac{d}{2}+\frac{n}{4})(\frac{n}{2d})^d\frac{2^{d-1}d^d}{n^{d+1}})=1.$$
This and the fact that
$$
Q_{\textbf{i},F}=Q_{\textbf{i},S} \geq \frac{5(2d)^d}{8n^d} -\frac{d(n/2d) 2^{d-1} d^d}{n^{d+1}} > (1/4)(2d/n)^d
$$
shows that $Q$ is a valid probability distribution.

As in Definition 3, we define the sequence $(j_{\textbf{i}},k_{\textbf{i}})_{\textbf{i} \in [\frac{n}{2d}]^d}$ where $j_{\textbf{i}}=(\textbf{i},F)$ and $k_{\textbf{i}}=(\textbf{i},S)$ to specify which halfcube of bins we are moving. Then we construct $F_{yes}^{\textbf{i}}$ and $F_{no}^{\textbf{i}}$ by choosing proper $A_{\textbf{i}},g_{\textbf{i}},m$:

In particular, let $m$ be the smallest odd integer that is larger than $C\log(1/\eps)$ where $C$ is a sufficiently large constant. The fact that $n <\frac{d}{(C^2\log\frac{1}{\epsilon})^3\epsilon}$ and $d<(C^2\log\frac{1}{\epsilon})^3$ imply that $nm^3 < \frac{d}{C\epsilon}$, and $\frac{2^{d-3}d^{d+1}}{n^{d+1}} > \frac{2^{d+2}m^3 d^d\epsilon}{n^d}$. We take $A=\frac{2^{d+2}m^3 d^d\epsilon}{n^d}$, noting that $\frac{2^{d-3}d^{d+1}}{n^{d+1}} > A$, and let $A_{\textbf{i}}=A$ for all $\textbf{i}$. Assuming $\epsilon>0$ is sufficiently small (as otherwise there is nothing to prove), we may assume that $n$ and $m$ are at least sufficiently large constants. Let $g_{\textbf{i}}=\cos(\frac{\pi}{m})$ for all $\textbf{i}$.

As the $A_{\textbf{i}}$ and $g_{\textbf{i}}$ are the same for all $\textbf{i}$, we refer to $F_{yes/no}^{A_{\textbf{i}},g_{\textbf{i}},m}$ simply as $F_{yes/no}$ and we note that Lemma \ref{F bounds lem} applies to them.
Noting that $A < \min(Q_{\textbf{i},F},Q_{\textbf{i},S})/(1+|g|)$, we can invoke Definition \ref{D yes no def} to define ensembles  $C_{yes}$ and $C_{no}$ over the set of halfcubes. Using these we construct our actual hard instances over $[n]^d$ as follows:
\begin{defn}
We define ensembles $D_{yes/no}$ of distributions over $[n]^d$ in the following way:
To sample a distribution $q$ from $D_{yes/no}$: first get a sample distribution $p$ over halfcubes from $C_{yes/no}$, one then takes a sample from $q$ by first sampling a halfcube using $p$ and then returning a uniform random sample from that halfcube.
\end{defn}
Note that in this construction, any distribution $q \sim D_{yes/no}$ is uniform inside each halfcube. And if we consider a distribution $p$ over $[\frac{n}{2d}]^d \times \{F,S\}$ where $p_{\textbf{i},F}=\frac{(2d)^d}{2} q_{\textbf{j}}$ where $\textbf{j} \in J_{\textbf{i}}$ and $p_{\textbf{i},S}=\frac{(2d)^d}{2} q_{\textbf{j}}$ where $\textbf{j} \in K_{\textbf{i}}$, we have $p \in C_{yes/no}$ by definition. Since one can produce a sample from $q$ given a sample from $p$, the statistical task of distinguishing whether a distribution $q$ was taken from $D_{yes}$ or $D_{no}$ in $N$ samples is equivalent to the task of distinguishing whether a distribution $p$ was taken from $C_{yes}$ or $C_{no}$ in $N$ samples. We will show by Corollary \ref{main bound cor} that this latter task is hard unless $N$ is large, but first we need to prove that a distribution in $D_{yes}$ is monotone with probability 1.
\begin{lem}
A distribution $q$ drawn from $D_{yes}$ is monotone with probability 1.
\end{lem}
\begin{proof}
Firstly, we note that since we are separating halfcubes based on the magnitude of its first coordinate, for any pair of bins $\textbf{j}$ and $\textbf{k}$ in the $\textbf{i}$th cube where $\textbf{j}<\textbf{k}$, $\textbf{j}$ and $\textbf{k}$ are either in the same halfcube or $\textbf{j}$ is in the first halfcube and $\textbf{k}$ in the second halfcube. If they are in the same halfcube, $q_{\textbf{j}}=q_{\textbf{k}}$ by construction. For the case that $\textbf{j}$ is in the first halfcube and $\textbf{k}$ is in the second halfcube, we show that $q_{\textbf{j}} \geq q_{\textbf{k}}$ by proving that within each cube, any bin in the first halfcube is always heavier than any bin in the second halfcube.  Note that a bin in the first halfcube in the $\textbf{i}$th cube has weight $\frac{Q_{\textbf{i},F} + \delta_{\textbf{i}}}{2^{d-1}d^d}$ and the one in the second halfcube in the $\textbf{i}$th pair has weight $\frac{Q_{\textbf{i},S} - \delta_{\textbf{i}}}{2^{d-1}d^d}$. Since $\delta_{\textbf{i}} \geq 0$ for $\delta_{\textbf{i}}$ drawn from $F_{yes}$ (by Lemma \ref{F bounds lem}) and $Q_{\textbf{i},S}=Q_{\textbf{i},F}$ for all $\textbf{i}$, $\frac{Q_{\textbf{i},F} + \delta_{\textbf{i}}}{2^{d-1}d^d} \geq \frac{Q_{\textbf{i},S} - \delta_{\textbf{i}}}{2^{d-1}d^d}$ always holds. So for any pair of bins $\textbf{j}$ and $\textbf{k}$ in the $\textbf{i}$th cube where $\textbf{j}<\textbf{k}$, $q_{\textbf{j}} \geq q_{\textbf{k}}$.

Secondly, we want to make sure the distribution is monotone across cubes. Note that there exists bins in the $\textbf{i}$th cube that are comparable to some bins in the $\textbf{j}$th cube for $\textbf{j}\neq \textbf{i}$ if and only if $\textbf{i}_a \leq \textbf{j}_a$ for $1 \leq a \leq d$ and there exists $a$ such that $\textbf{i}_a < \textbf{j}_a$. As the previous paragraph implies that bins in the first half of each cube are heavier than bins in the second half, it suffices to prove that any bin in the second halfcube of the $\textbf{i}$th cube (with weight $\frac{Q_{\textbf{i},S} - \delta_{\textbf{i}}}{2^{d-1}d^d}$) is heavier than any bin in the first halfcube in the $\textbf{j}$th cube (with weight $\frac{Q_{\textbf{j},F}+\delta_{\textbf{j}}}{2^{d-1}d^d}$).

We note that the difference,
$$ \frac{Q_{\textbf{i},S} - \delta_{\textbf{i}}}{2^{d-1}d^d} - \frac{Q_{\textbf{j},F}+\delta_{\textbf{j}}}{2^{d-1}d^d}$$
equals
$$\frac{1}{2^{d-1}d^d}\left[ \left( \frac{5(2d)^d}{8n^d} + \frac{2^dd^{d+1}}{4n^{d+1}} - \frac{(\textbf{i}_1+\textbf{i}_2+...\textbf{i}_d)2^{d-1}d^d}{n^{d+1}}\right) - \left(\frac{5(2d)^d}{8n^d} + \frac{2^dd^{d+1}}{4n^{d+1}} - \frac{(\textbf{j}_1+\textbf{j}_2+...\textbf{j}_d)2^{d-1}d^d}{n^{d+1}}\right)-\delta_{\textbf{i}}-\delta_{\textbf{j}} \right]  $$
$$ \geq  \frac{1}{2^{d-1}d^d}\left(\frac{2^{d-1}d^{d+1}}{n^{d+1}}-\delta_{\textbf{i}}-\delta_{\textbf{j}}\right)$$

By Lemma \ref{F bounds lem}, we have $\delta_{\textbf{i}} < \frac{2^{d-2}d^{d+1}}{n^{d+1}}$ for $\delta_{\textbf{i}} \sim F_{yes}$, so $\delta_{\textbf{i}}+\delta_{\textbf{j}} < \frac{2^{d-1}d^{d+1}}{n^{d+1}}$, which completes the proof that all distributions in $D_{yes}$ are monotone.
\end{proof}
In contrast to distributions in $D_{yes}$, the distributions in $D_{no}$ is at least $\epsilon$ far from monotone with high probability. To show this, we first need to prove a lemma about how far a distribution which is uniform over halfcubes is from monotone.
\begin{lem}\label{high dim monotone distance lem}
For a distribution $q$ uniform within each halfcube and $p$ an arbitrary monotone distribution, $\dtv(p,q) \geq \frac{1}{2}\sum_{\textbf{i} \in [\frac{n}{2d}]^d} \gamma_{\textbf{i}}$, where $\gamma_{\textbf{i}} = (\frac{2d-1}{2d})^{d-1}\max(0,\sum_{\textbf{j} \in K_{\textbf{i}}} q_{\textbf{j}}-\sum_{\textbf{j} \in J_{\textbf{i}}} q_{\textbf{j}})$
\end{lem}
\begin{proof}
Define $S_{\textbf{i},1}=\{\textbf{j}:\textbf{j} \in J_{\textbf{i}} \textrm{ and }\textbf{j}_a \neq 2d\textbf{i}_a, 2 \leq a \leq d\}$ and $S_{\textbf{i},2}=\{\textbf{j}:\textbf{j} \in K_{\textbf{i}} \textrm{ and } j_a \neq 2d\textbf{i}_a-2d+1, 2 \leq a \leq d \}$, we can pair a bin $\textbf{j} \in S_{\textbf{i},1}$ to $\textbf{k} \in S_{\textbf{i},2}$ where $\textbf{j}+(d,1,1,\ldots,1)=\textbf{k}$, so that for each such pair $\textbf{j}<\textbf{k}$. The probability that a random bin in the first halfcube of the $\textbf{i}$th cube lies in $S_{\textbf{i},1}$ is $(\frac{2d-1}{2d})^{d-1}$. Similarly, the probability that a random bin in the second halfcube of the $\textbf{i}$th cube lies in $S_{\textbf{i},2}$ is $(\frac{2d-1}{2d})^{d-1}$. Given that $q$ is uniform within each halfcube, we can get $\gamma_{\textbf{i}}=\max(0,\sum\limits_{\textbf{j} \in S_{\textbf{i},2}} q_{\textbf{j}}-\sum\limits_{\textbf{j} \in S_{\textbf{i},1}} q_{\textbf{j}})$.

Note that since $S_{\textbf{i},1} \cup S_{\textbf{i},2} \subset J_{\textbf{i}} \cup K_{\textbf{i}}$,
$$\dtv(p,q)=\frac{1}{2}\sum\limits_{\textbf{i} \in [\frac{n}{2d}]^d} \sum\limits_{\textbf{j} \in J_{\textbf{i}} \cup K_{\textbf{i}}} |p_{\textbf{j}}-q_{\textbf{j}}| \geq \frac{1}{2}\sum\limits_{\textbf{i} \in [\frac{n}{2d}]^d} \sum\limits_{\textbf{j} \in S_{\textbf{i},1} \cup S_{\textbf{i},2}} |p_{\textbf{j}}-q_{\textbf{j}}| \geq \frac{1}{2} \sum_{\textbf{i} \in [\frac{n}{2d}]^d}\left(\left|\sum\limits_{\textbf{j} \in S_{\textbf{i},1}} p_{\textbf{j}}-\sum\limits_{\textbf{j} \in S_{\textbf{i},1}} q_{\textbf{j}}\right|+\left|\sum\limits_{\textbf{j} \in S_{\textbf{i},2}} p_{\textbf{j}}-\sum\limits_{\textbf{j} \in S_{\textbf{i},2}} q_{\textbf{j}}\right|\right)$$

It suffices to prove that $\left(\left|\sum\limits_{\textbf{j} \in S_{\textbf{i},1}} p_{\textbf{j}}-\sum\limits_{\textbf{j} \in S_{\textbf{i},1}} q_{\textbf{j}}\right|+\left|\sum\limits_{\textbf{j} \in S_{\textbf{i},2}} p_{\textbf{j}}-\sum\limits_{\textbf{j} \in S_{\textbf{i},2}} q_{\textbf{j}}\right|\right) \geq \gamma_{\textbf{i}}$ for all $\textbf{i}$. Given that $p$ is monotone and $\textbf{k}>\textbf{j}$, $p_{\textbf{j}} \geq p_{\textbf{k}}$ for all pairs of $\textbf{j}$ and $\textbf{k}$ in $S_{\textbf{i},1}$ and $S_{\textbf{i},2}$ with $\textbf{k} = \textbf{j}+(d,1,\ldots,1)$. Therefore, summing over $S_{\textbf{i},1}$ and $S_{\textbf{i},2}$, $\sum\limits_{\textbf{j} \in S_{\textbf{i},1}} p_{\textbf{j}} \geq \sum\limits_{\textbf{j} \in S_{\textbf{i},2}} p_{\textbf{j}}$ for all $\textbf{i}$.

If $\sum_{\textbf{j} \in S_{\textbf{i},1}} q_{\textbf{j}}-\sum\limits_{\textbf{j} \in S_{\textbf{i},2}}q_{\textbf{j}} \geq 0$, then $\gamma_{\textbf{i}} = 0$, and we have our desire inequality.

If $\sum_{\textbf{j} \in S_{\textbf{i},1}} q_{\textbf{j}}-\sum\limits_{\textbf{j} \in S_{\textbf{i},2}}q_{\textbf{j}} < 0$, then $\gamma_{\textbf{i}}=\sum\limits_{\textbf{j} \in S_{\textbf{i},2}} q_{\textbf{j}}-\sum\limits_{\textbf{j} \in S_{\textbf{i},1}}q_{\textbf{j}}>0$.
In this case,
\begin{align*}
|\sum_{\textbf{j} \in S_{\textbf{i},1}} p_{\textbf{j}}-\sum_{\textbf{j} \in S_{\textbf{i},1}} q_{\textbf{j}}|+|\sum_{\textbf{j} \in S_{\textbf{i},2}} p_{\textbf{j}}-\sum_{\textbf{j} \in S_{\textbf{i},2}} q_{\textbf{j}}| & \geq (\sum_{\textbf{j} \in S_{\textbf{i},1}} p_{\textbf{j}}-\sum_{\textbf{j} \in S_{\textbf{i},1}} q_{\textbf{j}})+(\sum_{\textbf{j} \in S_{\textbf{i},2}} q_{\textbf{j}}-\sum_{\textbf{j} \in S_{\textbf{i},2}} p_{\textbf{j}})\\
& = (\sum_{\textbf{j} \in S_{\textbf{i},1}} p_{\textbf{j}}-\sum_{\textbf{j} \in S_{\textbf{i},2}} p_{\textbf{j}})+(\sum\limits_{\textbf{j} \in S_{\textbf{i},2}} q_{\textbf{j}}-\sum_{\textbf{j} \in S_{\textbf{i},1}} q_{\textbf{j}}) \geq \gamma_{\textbf{i}}.
\end{align*}
Summing this inequality over all $\textbf{i}$, we have that $\dtv(p,q) \geq \frac{1}{2}\sum_{\textbf{i} \in [\frac{n}{2d}]^d}\gamma_{\textbf{i}}$.
\end{proof}
This lemma is a multidimensional analogue of Lemma \ref{monotone distance lem}. It gives a lower bound of distance from monotone for a distribution uniform within each halfcube, which helps us prove that a random distribution from $D_{no}$ is not close to being monotone. We will prove in the next lemma that it's at least $\epsilon$ far from monotone with probability at least $99 \%$.
\begin{lem}
With $99\%$ probability, a random distribution drawn from $D_{no}$ is at least $\epsilon$ far from monotone.
\end{lem}
\begin{proof}
By the definition of $D_{no}$, if $q$ is sampled from $D_{no}$, we have that
\begin{align*}
\sum_{\textbf{j} \in K_{\textbf{i}}} q_{\textbf{j}}-\sum_{\textbf{j} \in J_{\textbf{i}}} q_{\textbf{j}} & = p_{\textbf{i},S} - p_{\textbf{i},F} = \\
& = (Q_{\textbf{i},S}-\delta_{\textbf{i}} ) - (Q_{\textbf{i},F}+\delta_{\textbf{i}} )\\
& = -2\delta_{\textbf{i}}.
\end{align*}
Where $p$ is the corresponding distribution from $C_{no}$. Therefore, in the notation of Lemma \ref{high dim monotone distance lem} we have that
$$
\gamma_{\textbf{i}} = \begin{cases} 0 & \textrm{if }\delta_{\textbf{i}} \geq 0\\ -2\left(\frac{2d-1}{2d} \right)^{d-1}\delta_{\textbf{i}} & \textrm{if }\delta_{\textbf{i}} < 0.\end{cases}
$$
Note that $\left(\frac{2d-1}{2d} \right)^{d-1} = \frac{1}{(1+1/(2d-1))^{d-1}} > e^{-1/2}$. Thus, by Lemma \ref{F bounds lem}, this means that $\gamma_{\textbf{i}}$ is non-zero independently with probability $1/m$ and if it is non-zero, it is at least $2^{d+2}m d^d \eps/n^d$.

Letting $X$ be the number of non-zero $\gamma_{\textbf{i}}$'s. It is distributed as $\mathrm{Bin}((n/2d)^d,1/m)$ and so with probability at least $99\%$ is at least $(n/2d)^d/(2m)$. In such a case we have that the distance of $q$ from uniform is at least
\begin{align*}
\frac{1}{2}\sum_{\textbf{i} \in [\frac{n}{2d}]^d}\gamma_{\textbf{i}} & \geq X 2^{d+1}m d^d \eps/n^d > \eps.
\end{align*}

\end{proof}
We have proved that a distribution in $D_{yes}$ is monotone with probability 1 and a distribution in $D_{no}$ is $\epsilon$ far from monotone with $99\%$ probability. In the next section, we will apply Proposition \ref{main bound prop} to use this to show that one cannot build a monotonicity tester with too few samples.

\subsection{Lower Bound of Multidimensional Monotonicity Testing}

Here we prove Theorem \ref{monotonicity thm} starting with the case where $n$ is at most $\frac{d}{(C^2\log\frac{1}{\epsilon})^3\epsilon}$ and is a multiple of $2d$. Let $N$ be a sufficiently small multiple of $(n/2d)^d (1/\eps)^2 / (d m^6 \log(1/\eps))$ and suppose for sake of contradiction that there is a tester that tests monotonicity over $[n]^d$ with $N$ samples. As a distribution from $D_{yes}$ is monotone and a distribution from $D_{no}$ is $\eps$-far from monotone with $99\%$ probability, this tester can reliably distinguish $N$ samples from a distribution from $D_{yes}$ from $N$ samples from a distribution from $D_{no}$. However, this is equivalent to distinguishing $C_{yes}$ from $C_{no}$, which is difficult by Proposition \ref{main bound prop}.

In particular, in the context of Corollary \ref{main bound cor}, we have that $B = O(N (n/2d)^d)$ and $x_{max} = O(m^3\eps) < 1/10$. This makes $Bx_{max}^2$ at most a small multiple of $1/(d\log(1/\eps)) < \log(s)$. Thus, we can apply Corollary \ref{main bound cor} and conclude that $\dtv(C_{yes}^N,C_{no}^N) < 1/100$ and thus that one cannot distinguish the two with $N$ samples, providing our contradiction.

For $n$ not of the desired form, let $n_0$ be the largest integer smaller than $n$ that is both a multiple of $2d$ and at most $\frac{d}{(C^2\log\frac{1}{\epsilon})^3\epsilon}$. As monotonicity testing over $[n_0]^d$ is a special case of monotonicity testing over $[n]^d$, we obtain a lower bound of
$$
\Omega((n_0/2d)^d (1/\eps)^2 / (d m^6 \log(1/\eps))) = 2^{-O(d)}d^{-d} \eps^{-2} \log^{-7}(1/\eps) \min(n,d \eps^{-1} \log^{-3}(1/\eps))^d.
$$
This completes our proof.

\section{Log Concavity Distribution Testing}\label{log concavity sec}
To prove our lower bound for log-concavity testing, we will again use Proposition \ref{main bound prop} to construct indistinguishable ensembles $D_{yes}$ and $D_{no}$. In particular, we need to carefully instantiate our construction so that distributions from $D_{yes}$ are log-concave almost surely, while distributions from $D_{no}$ are $\epsilon$ far with high probability. Then Proposition 1 will imply that these ensembles are indistinguishable without a large number of samples, giving us a desired lower bound.

\subsection{Construction}

The intuition for log concavity testing over $[n]$ is: start with a log concave base distribution $Q$ over $[n]$ and separate it into groups of 6 bins, then modify the 2nd and 5th bin in each group. The reason we choose to move bins in this way is that with only 1 bin being moved in each triple, it's easier to evaluate how each move affects log concavity. As long as the size of the move is small enough, the distribution will still be log concave. On the other hand, large moves cause it to be $\epsilon$ far from log concavity.

We begin by constructing the base distribution $Q$ over $[n]$. Firstly, we assume that with $C\log(1/\eps) < n <\frac{1}{C^2\epsilon^{\frac{1}{2}}(\log \frac{1}{\epsilon})^{\frac{3}{2}}}$ for some sufficiently large constant $C$ and that $n$ is a multiple of $6$. We let $Q_i=\frac{b}{n}e^{-(\frac{i}{n})^2}$ with $b=\frac{n}{\sum\limits_{i=1}^{n}e^{-(\frac{i}{n})^2}}$. Observe that $Q_i^2=\frac{b^2}{n^2}e^{-\frac{2i^2}{n^2}}<\frac{b^2}{n^2}e^{-\frac{2i^2+2}{n^2}}=Q_{i-1}Q_{i+1}$, we use this $Q$ as it is in some sense roughly the most log concave that it can be and $b \in (1,e)$ is a normalization factor that ensures $\sum\limits_{i=1}^{n} Q_i =1$. Additionally, $Q_i>Q_j$ for $i<j$. Let the sequence $(j_i,k_i)_{1 \leq i \leq \frac{n}{6}}$ be defined by $j_i=6i-4$ and $k_i=6i-1$.

Let $m$ be the smallest odd integer larger than $C\log n$. Since $n<\frac{1}{48C^{\frac{3}{2}}\epsilon^{\frac{1}{2}}(\log \frac{1}{\epsilon})^{\frac{3}{2}}}$, we have $m < C \log \frac{1}{\epsilon}$. For $1 \leq i \leq \frac{n}{6}$, we take $$C_i=n^3(Q_{6i-1}-\sqrt{Q_{6i}Q_{6i-2}})=bn^2e^{\frac{36i^2-12i+1}{n^2}}(1-e^{-\frac{1}{n^2}})$$
Given $\epsilon>0$ sufficiently small, we can assume that $n$ and $m$ are bigger than sufficiently large constants. Therefore, we have
$$\frac{1}{2n^2}<1-e^{-\frac{1}{n^2}}<\frac{1}{n^2}.$$
Thus, $C_i=\Theta(1)$.

Next, we define the sequence of swapped bins by $(j_i,k_i) = (6i-4,6i-1)$ for $1\leq i \leq n/6$. Finally, we define $A_i$ and $g_i$ so that $F_{yes}$ and $F_{no}$ are given by $C_i/n^3 - Cm^3\eps/n (\cos(\frac{2\pi a}{m}) + \cos(\frac{\pi a}{m}))$ and $C_i/n^3 - Cm^3\eps/n (\cos(\frac{2\pi (a+\frac{1}{2})}{m}) + \cos(\frac{\pi a}{m}))$, respectively. It will be useful to compare this to another construction producing the same result. Namely:
$$
Q_a' := \begin{cases}Q_a & \textrm{if } i \not \equiv 1 \pmod{3}\\ Q_a+ C_i/n^3 & \textrm{if } x = 6i-4 \\ Q_a - C_i/n^3 & \textrm{if } x = 6i-1. \end{cases}
$$
Then we can likewise construct $D_{yes/no}$ from $Q'$ using the same sequence of $(j_i,k_i)$ and letting $A_i = -Cm^3\eps/n$ and $g_i = - \cos(\frac{\pi a}{m})$ for all $i$. We will switch back and forth between these two interpretations as necessary. We note that the $Q_i'$ are all $\Theta(1/n)$ and therefore the $x_{max}$ in the primed interpretation of the construction is $O(Cm^3\eps)$.

We now have some important properties to prove about this construction. Namely that distributions from $D_{yes}$ are log-concave, distributions from $D_{no}$ are far from that and that the two are indistinguishable with few samples. To begin:

\begin{lem}
Any distribution $p$ in $D_{yes}$ is log concave with probability 1.
\end{lem}
\begin{proof}
Consider a distribution $p \sim D_{yes}$, given a sample of $\delta_i \in F_{yes}^i$, it moves $(6i-4)$th bin up by $\delta_i$ and $(6i-1)$th bin down by $\delta_i$. We note that applying Lemma \ref{F bounds lem} to the $Q'$ formulation we have that $\delta_i \leq C_i/n^3$. On the other hand, as $n^2 \ll C^{-1} (1/\eps) (1/m^3),$ we have that $\delta_i > 0$ for all $i$. Using this, we can check log-concavity of $p$ at each $i$ based on $i\pmod 6$. In particular,
\begin{align*}
p_{6i-5}^2 - p_{6i-4}p_{6i-6} & = Q_{6i-5}^2 - Q_{6i-6}(Q_{6i-4}+\delta_i) \\
& = Q_{6i-5}^2 - Q_{6i-6}Q_{6i-4} - \delta_iQ_{6i-6}\\
& = \Omega(1/n^3) - \Omega(1/n^4) > 0.
\end{align*}
We are similarly still log-concave at $6i-3$. We also have that
\begin{align*}
p_{6i-1} - \sqrt{p_{6i}p_{6i-2}} & = Q_{6i-1} -\delta_i - \sqrt{Q_{6i}Q_{6i-2}}\\
& = C_i/n^3 - \delta_i \geq 0.
\end{align*}
The other locations follow immediately from $\delta_i >0$ as
$$
p_{6i-4}^2 - p_{6i-3}p_{6i-5} > Q_{6i-4}^2 - Q_{6i-3}Q_{6i-5} > 0,
$$
and similarly for $6i-2$ and $6i$.
\end{proof}

In order to show that $D_{no}$ is likely far from log-concave we need a Lemma to allow us to show how far a distribution is from log-concave.
\begin{lem}\label{distance from log concave lem}
For a distribution $p$ over $[n]$ with $p_{3i-2} > p_{3i}> \frac{4}{5}p_{3i-2}$, $p_{3i-1}>\frac{3}{4}p_{3i}$ for all $1 \leq i \leq \frac{n}{3}$, and $q$ any log concave distribution over $[n]$, then $\dtv(p,q) \geq \frac{1}{2}\sum_{i=1}^{\frac{n}{3}} \max (0, \sqrt{p_{3i-2}p_{3i}}-p_{3i-1})$.
\end{lem}
\begin{proof}
It's sufficient to show that $|p_{3i-2}-q_{3i-2}|+|p_{3i-1}-q_{3i-1}|+|p_{3i}-q_{3i}| \geq \max (0, \sqrt{p_{3i-2}p_{3i}}-p_{3i-1})$ for all $1 \leq i \leq \frac{n}{3}$. For a log concave $q$, $\sqrt{q_{3i-2}q_{3i}}-q_{3i-1}\leq  0$ must hold for all $i$. Fix a given $i$, given $p$ a distribution over $[n]$,  we will find such $q_{3i-2}, q_{3i-1},q_{3i}$ that minimizes $|p_{3i-2}-q_{3i-2}|+|p_{3i-1}-q_{3i-1}|+|p_{3i}-q_{3i}|$ subject to the constraint $q_{3i-1}\geq\sqrt{q_{3i}q_{3i-2}}$. If $q_{3i-2} > p_{3i-2}$,  we get a better set of $q_{3i-2}, q_{3i-1},q_{3i}$ by taking $q_{3i-1}=p_{3i-2}$ and keeping $q_{3i-1},q_{3i}$ unchanged. Similar reasoning applies to the case of $q_{3i-1} < p_{3i-1}$ or $q_{3i} > p_{3i}$. Therefore, we have that the optimal $q_{3i-2}, q_{3i-1},q_{3i}$ satisfy $q_{3i-2} \leq p_{3i-2}, q_{3i-1} \geq p_{3i-1}, q_{3i} \leq p_{3i}$, so
$$|p_{3i-2}-q_{3i-2}|+|p_{3i-1}-q_{3i-1}|+|p_{3i}-q_{3i}|=p_{3i-2}-q_{3i-2}+q_{3i-1}-p_{3i-1}+p_{3i}-q_{3i}.$$
Let $\vec{p}=(p_{3i-2},p_{3i-2},p_{3i})$ and $\vec{q}=(q_{3i-2},q_{3i-2},q_{3i})$, $f:\mathbb{R}^3 \rightarrow \mathbb{R}$ is a function where $f(\vec{x})=\sqrt{x_{3i-2}x_{3i}}-x_{3i-1}$, applying the mean value theorem, we have
$$f(\vec{p})-f(\vec{q})=(\vec{p}-\vec{q}) \cdot \nabla f(\vec{x})$$
for some $\vec{x}=(x_1,x_2,x_3)$ between $\vec{p}$ and $\vec{q}$. In particular, it's clear that $q_{3i-2} \leq x_1 \leq p_{3i-2}$, $p_{3i-1} \leq x_2 \leq q_{3i-1}$ and $q_{3i} \leq x_3 \leq p_{3i}$. Expanding the dot product, we can get
$$f(\vec{p})-f(\vec{q})=\frac{1}{2}\sqrt{\frac{x_3}{x_1}}(q_{3i-2}-p_{3i-2})+(p_{3i-1}-q_{3i-1})+\frac{1}{2}\sqrt{\frac{x_1}{x_3}}(q_{3i}-p_{3i}).$$
Notice that if $q_{3i-2}>\frac{p_{3i-2}}{2}$ and $q_{3i}>\frac{p_{3i}}{2}$, using the relation $p_{3i-2} > p_{3i}> \frac{4}{5}p_{3i-2}$, we have $x_1 \leq p_{3i-2} \leq 2p_{3i} <4q_{3i} \leq 4x_3$ and $x_3 \leq p_{3i} \leq 2p_{3i-2} <4q_{3i-2} \leq 4x_1$, we have $f(\vec{p})-f(\vec{q}) \leq |p_{3i-2}-q_{3i-2}|+|p_{3i-1}-q_{3i-1}|+|p_{3i}-q_{3i}|$ since $\frac{1}{2}\sqrt{\frac{x_3}{x_1}},\frac{1}{2}\sqrt{\frac{x_1}{x_3}}<1$.

On the other hand, if $q_{3i-2} \leq \frac{p_{3i-2}}{2}$, by manipulating the relations between $p_{3i-2},p_{3i-1},p_{3i}$ in the hypothesis, we have
$$|p_{3i-2}-q_{3i-2}|+|p_{3i-1}-q_{3i-1}|+|p_{3i}-q_{3i}| \geq |p_{3i-2}-q_{3i-2}| \geq \frac{p_{3i-2}}{2} \geq p_{3i-2}-\frac{5}{8}p_{3i}>p_{3i-2}-\frac{5}{6}p_{3i-1}>\sqrt{p_{3i-2}p_{3i}}-p_{3i-1}.$$
Similarly, with the case that $q_{3i} \leq \frac{p_{3i}}{2}$, we can show that
$$|p_{3i-2}-q_{3i-2}|+|p_{3i-1}-q_{3i-1}|+|p_{3i}-q_{3i}| \geq |p_{3i}-q_{3i}| \geq \frac{p_{3i}}{2} = (\frac{5}{4}-\frac{3}{4})p_{3i} \geq \frac{5}{4}p_{3i}-p_{3i-1}>\sqrt{p_{3i-2}p_{3i}}-p_{3i-1}.$$

So in any case we have that
$$
|p_{3i-2}-q_{3i-2}|+|p_{3i-1}-q_{3i-1}|+|p_{3i}-q_{3i}| \geq \sqrt{p_{3i-2}p_{3i}}-p_{3i-1}.
$$
Summing over $i$ yields our result.
\end{proof}

Using this we show that a distribution from $D_{no}$ is likely far from log-concave.
\begin{lem}
With $99\%$ probability, a random distribution drawn from $D_{no}$ is $\epsilon$ far from log concave.
\end{lem}
\begin{proof}
Let $p$ be taken from $D_{no}$ and let $q$ be an arbitrary log-concave distribution over $[n]$. Applying Lemma \ref{F bounds lem}, we note that for each $i$ there is independently a $1/m$ probability that $\delta_i = C_i/n^3 + \Omega(Cm\eps/n).$ Let $X$ be the number of such $i$'s. We note that for each such $i$ that
\begin{align*}
p_{6i-1} - \sqrt{p_{6i}p_{6i-2}} & = \\
& = Q_{6i-1} - \sqrt{Q_{6i}Q_{6i-2}}- \delta_i\\
& = C_i/n^3 - \delta_i < -\Omega(Cm\eps/n).
\end{align*}
Therefore, by Lemma \ref{distance from log concave lem}, $\dtv(p,q) \geq \Omega(XCm\eps/n).$ It thus suffices to show that with $99\%$ probability that $X = \Omega(n/m)$. However, as $X\sim\mathrm{Bin}(n/6,1/m)$, this is clear.
\end{proof}

Finally, we can complete our proof of Theorem \ref{log concavity thm}.
\begin{proof}
We begin by proving it for $n$ less than $\frac{1}{C^2\epsilon^{\frac{1}{2}}(\log \frac{1}{\epsilon})^{\frac{3}{2}}}$ and a multiple of $6$. Suppose that there is a tester that can reliably distinguish between a log-concave distribution and one that is $\eps$-far using $N < C^{-3}n \eps^{-2} \log^{-7}(1/\eps)$ samples. As a distribution from $D_{yes}$ is log-concave and one from $D_{no}$ is likely $\eps$-far our tester can reliably distinguish the two. However, using the $Q'$ interpretation of our construction, we have $B=O(N/n)$ so $Bx_{max}^2 = O(C^2m^6\eps^2 N/n) = O(C^{-1}/\log(n/\eps x_{max})).$ Thus, we can apply Corollary \ref{main bound cor} to see that $\dtv(D_{yes}^N,D_{no}^N)<1/100$ which contradicts our algorithm being able to reliably distinguish them.

For other $n$, we can just apply this result to $n_0$, the largest integer that satisfies our conditions. As a log-concave distribution on $[n_0]$ is also a log-concave distribution over $[n]$, this gives a reduction between the testing problems and completes the proof.
\end{proof}

\section{Conclusion}

In this paper we have produced a general framework for proving distribution testing lower bounds for properties defined by local inequalities between the individual bin probabilities. Applying it requires finding an instantiation of our construction so that many bins satisfy these inequalities tightly and changing their values slightly will break the property in question. Usually, this technique should give lower bounds comparable to the testing-by-learning algorithm of $O(n/\eps^2)$ samples up to logarithmic factors so long as $n$ is not too big, while for larger values of $n$ it will often fail to find further improvements.

As applications of this new technique we have proved new lower bounds for monotonicity testing, and nearly optimal lower bounds for log-concavity testing.

\end{flushleft}

\end{document}